\documentclass{article} %
\usepackage[table]{xcolor}
\usepackage{iclr2022_conference,times}

\usepackage[pagebackref=true,breaklinks=true,letterpaper=true,colorlinks,citecolor=blue,bookmarks=false]{hyperref}
\usepackage[labelfont=bf]{caption}
\usepackage{url}
\usepackage{amsmath,amsfonts,bm}
\usepackage{multirow}
\usepackage{booktabs} %
\usepackage{bbm}
\usepackage{enumitem}
\usepackage{amsthm}
\usepackage{amssymb}
\usepackage{mathtools}
\usepackage{lib}
\newtheorem{theorem}{Theorem}[section]

\newtheorem{lemma}[theorem]{Lemma}
\usepackage{lib}            %
\usepackage{alias}          %
\usepackage{wrapfig}          %
\usepackage{algorithm} 
\usepackage{algpseudocode} 
\usepackage{bm}

\title{Learning Value Functions from \\ Undirected State-only Experience}

\author{Matthew Chang%
\thanks{~denotes equal contribution. Project website:
{\url{https://matthewchang.github.io/latent_action_qlearning_site/}.}} \; \; 
Arjun Gupta\hbox to 0pt{$^*$} \; \; 
Saurabh Gupta \\
University of Illinois at Urbana-Champaign\\
\texttt{\{mc48, arjung2, saurabhg\}@illinois.edu}\\
}

\iclrfinalcopy %
\begin{document}

\maketitle

\begin{abstract}
This paper tackles the problem of learning value functions from undirected
state-only experience (state transitions without action labels \ie $(s,s',r)$
tuples). We first theoretically characterize the applicability of Q-learning in
this setting. We show that tabular Q-learning in discrete Markov decision
processes (MDPs) learns the same value function under any arbitrary refinement
of the action space. This theoretical result motivates the design of Latent
Action Q-learning or LAQ, an offline RL method that can learn effective value
functions from state-only experience. Latent Action Q-learning (LAQ) learns
value functions using Q-learning on discrete latent actions obtained through a
latent-variable future prediction model. We show that LAQ can recover 
value functions that have high correlation with value functions learned using
ground truth actions. Value functions learned using LAQ lead to sample
efficient acquisition of goal-directed behavior, can be used with
domain-specific low-level controllers, and facilitate transfer across
embodiments. Our experiments in 5 environments ranging from 2D grid world to 3D
visual navigation in realistic environments demonstrate the benefits of LAQ
over simpler alternatives, imitation learning oracles, and competing methods.

\end{abstract}
\section{Introduction}
\seclabel{intro}
Offline or batch reinforcement learning focuses on learning goal-directed
behavior from pre-recorded data of undirected experience in the form of $(s_t,
a_t, s_{t+1}, r_t)$ quadruples. However, in many realistic applications, action
information is not naturally available (\eg when learning from video
demonstrations), or worse still, isn't even well-defined (\eg when learning
from the experience of an agent with a different embodiment).  Motivated by
such use cases, this paper studies if, and how, intelligent behavior can be
derived from undirected streams of observations: $(s_t, s_{t+1},
r_t)$.\footnote{We assume $r_t$ is observed. Reward can often be sparsely labeled in observation streams with low effort.}

At the face of it, it might seem that observation-only data would be useless towards learning goal-directed policies. 
After all, to learn such a policy, we need to know what actions to execute. 
Our key conceptual insight is that while an
observation-only dataset doesn't tell us the precise action to execute, \ie
the policy $\pi(a|s)$; it may still tell us which states are more likely to
lead us to the goal than not, \ie the value function $V(s)$. For example,
simply by looking at someone working in the kitchen, we can infer that
approaching the microwave handle is more useful (\ie has higher value) for
opening the microwave than to approach the buttons. Thus, we can still make use
of observation-only data, if we focused on learning value functions as opposed
to directly learning goal-directed policies. Once we have learned a good value
function, it can be used to quickly acquire or infer behavior.
Using learned value functions as dense rewards can lead to quick policy
learning through some small amount of interaction in the environment.
Alternatively, they could be used to directly guide the behavior of low-level
controllers that may already be available for the agent (as is often the case
in robotics) without any further training. Furthermore, decoupling the
learning of value functions from policy learning enables deriving behavior
for agents with a different embodiment as long as the overall solution strategy remains similar.
Thus, the central technical question is how to learn a good value function from
undirected observation streams. Is it even possible? If so, under what
conditions? This paper tackles these questions from a theoretical
and practical perspective.

We start out by characterizing the behavior of tabular Q-learning from
\cite{watkins1989learning} under missing action labels. We note 
that Q-learning with naively imputed action labels is equivalent to the TD(0) policy evaluation, which serves as a simple baseline method for deriving a value function. However, depending on the policy that generated
the data, the learned values (without any action grounding) can differ from the
optimal values. Furthermore, it is possible to construct simple environments
where the behavior implied by the learned value function is also sub-optimal.

Next, we present a more optimistic result. There are settings in which
Q-learning can recover the optimal value function even in the absence of the
knowledge of underlying actions. Concretely, we prove that if we are able to
obtain an action space which is a strict refinement of the original action
space, then Q-learning in this refined action space recovers the optimal
value function.

This motivates a practical algorithm for learning value functions from the
given undirected observation-only experience. We design a latent-variable
future prediction model that seeks to obtain a refined action space. 
It operates by predicting $s_{t+1}$ from $s_t$ and a
discrete latent variable $\hat{a}$ from a set of actions $\mathbf{\hat{A}}$
(\secref{action}). Training this latent variable model assigns a discrete
action $\hat{a}_t$ to each $(s_t, s_{t+1})$ tuple. This allows us to employ
Q-learning to learn good value functions (\secref{q}).  The learned value
function is used to derive behavior (\secref{behavior}) either through some
online interaction with the environment, or through the use of domain specific
low-level controllers.

The use of a latent action space for Q-learning allows us to exceed the performance of methods based on policy evaluation~\citep{edwards2019perceptual}, which will learn the value of the demonstration policy, not the optimal value function. Additionally, it side-steps the problem of reconstructing high-dimensional images faced by other state-only value learning methods~\citep{edwards2020estimating}. Other approaches for learning from state-only data rely on imitating the demonstration data, which renders them unable to improve on sub-optimal demonstration data. See \secref{related} for more discussion.

Our experiments in five environments (2D grid world, 2D continuous control,
Atari game Freeway, robotic manipulation, and visual navigation in realistic 3D
environments) test our proposed ideas. 
Our method approximates a refinement of the latent space better than clustering alternatives, and in turn, learns value functions highly correlated with ground truth.
Good value functions in-turn lead to sample efficient
acquisition of behavior, leading to significant improvement over
learning with only environment rewards.  Our method compares well against
existing methods that learn from undirected observation-only data, while being
also applicable to the case of high-dimensional observation spaces in the form
of \rgb images. We are also able to outperform imitation learning methods, even
when these imitation learning methods have access to privileged ground-truth
action information. Furthermore, our method is able to use observation-only
experience from one agent to speed up learning for another agent with a
different embodiment.

\section{Preliminaries}

Following the notation from~\cite{sutton2018rli}, our Markov decision
process (MDP) is specified by $(\mathbf{S},\mathbf{A},p,\gamma)$, where
$\mathbf{S}$ is a state space, $\mathbf{A}$ is an action space, $\gamma$ is the
discount factor, and $p(s', r | s , a)$ is the state/reward joint dynamics
function. It specifies the probability distribution that the agent ends up in
state $s'$, receives a reward of $r$ on executing action $a$ from state $s$.

Offline or batch RL~\citep{lange2012batch, levine2020offline} studies the
problem of deriving high reward behavior when only given a dataset of
experience in an MDP, in the form of a collection of quadruples $(s, a, s',
r)$. In this paper, we tackle a harder version of this problem where instead we
are only given a collection of triplets $(s, s', r)$, \ie experience without
information about intervening actions. In general, this dataset could contain
any quality of behavior. In contrast to some methods (see \secref{related}), we
will not assume that demonstrations in the dataset are of high quality, and
design our method to be robust to sub-optimal data. Using such a dataset, our
goal is to learn good value functions. A value function under a policy $\pi$ is
a function of states that estimates the expected return when starting in $s$
and following $\pi$ thereafter.

In this paper, we will focus on methods based on
Q-learning~\citep{watkins1989learning} for tackling this problem. Q-learning
has the advantage of being \textit{off-policy}, \ie, experience from another
policy (or task) can be used to learn or improve a different policy for a
different task. Q-learning seeks to learn the
optimal Q-function $Q^*(s,a)$ by iteratively updating $Q(s_t,a_t)$ to the Bellman equation. This process converges to the $Q^*$ under mild conditions in many
settings~\citep{watkins1989learning}.

\begin{figure*}
    \centering
    \includegraphics[width=\textwidth]{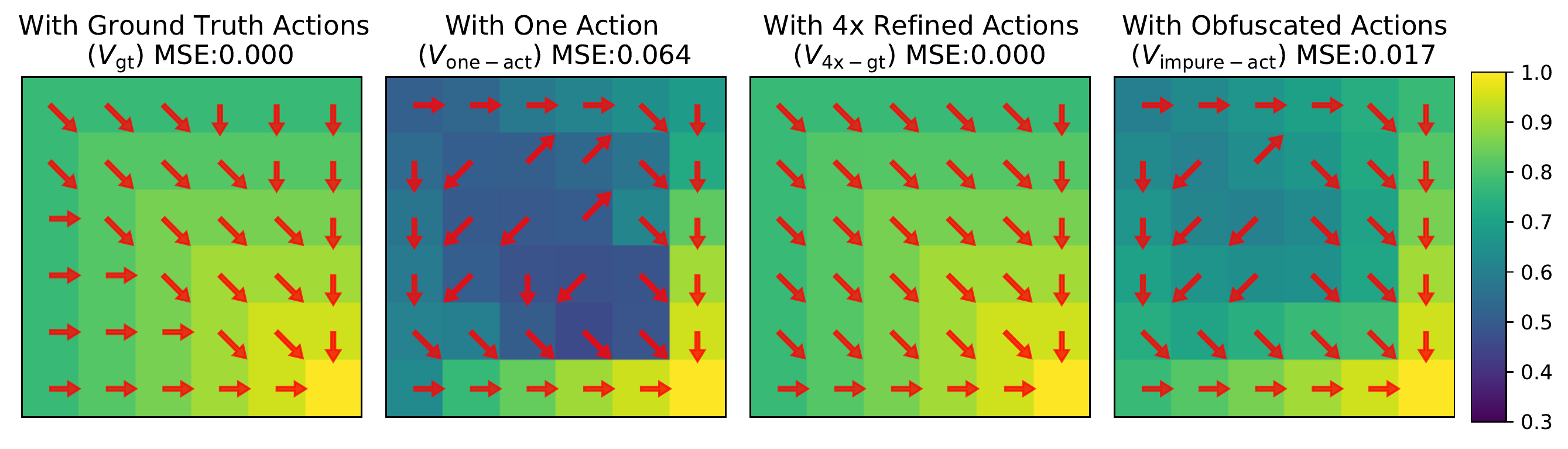}
    \vspace{-8mm}
    \caption{
     We visualize the learned value function when using different
     action labels for Q-learning: ground truth actions, one single action, a $4\times$ refined
     action space, and obfuscated actions. We also report the mean squared error
     from the optimal value function.
     Arrows show the behavior induced by the value function (picking neighboring state with highest value).
     We note that
     a) ignoring the intervening actions works poorly both in terms of value
     function estimates and the induced behavior, b) refined actions don't cause
     any degradation in performance, c) noisy actions that don't adhere the
     refinement constraint cause degradation in performance.
    See \secref{expval} for more details.}
    \vspace{-2mm}
    \figlabel{tabularexp}
\end{figure*}

\section{Characterizing Q-learning without True Action Labels}
\seclabel{theory}
We characterize the outcome of Q-learning in settings where we don't have
ground truth intervening actions in the offline dataset being used for
Q-learning. Without action labels, one could simply assign all transitions the
same label. In this case, Q-Learning becomes TD(0) policy evaluation. The
induced value function isn't the optimal value function for the MDP, but rather
the value according to the policy that generated the dataset. Depending
on the dataset, this could be sub-optimal.

Next, we study if labeling $(s,s',r)$ samples with actions from a different
action space $\mathbf{\hat{A}}$ to construct a new MDP could aid learning.
More specifically, can the optimal Q-function for this new MDP, as obtained
through Q-learning on samples $(s,s',r)$ labeled with actions from
$\mathbf{\hat{A}}$, be useful in the original MDP? We show that under the right
conditions the value function learned under the altered action space
$\mathbf{\hat{A}}$ is identical to the value function learned for the original
MDP.

\subsection{Optimality of Action Refinement}
Assume we have a Markov
Decision Process (MDP) $M$ specified by $(\mathbf{S},\mathbf{A},p,\gamma)$. Let the
action space $\mathbf{A}$ be composed of actions ${a_1, a_2, a_3,...,a_n} \in
\mathbf{A}$. We are interested in the value learned under a modified MDP,
$\hat{M}$ composed of $(\mathbf{S},\mathbf{\hat{A}},\hat{p},\gamma)$. We will show that
if the actions and transitions $\mathbf{\hat{A}}$ and $\hat{p}$ are a
  \textit{refinement} of $\mathbf{A}$ and $p$, then the value function learned
  on $\hat{M}$, $V_{\hat{M}}$ is identical to the value function learned on $M$,
  $V_{M}$.  
  We define actions and transitions in $\hat{M}$ to be a refinement
  of those in $M$ when, a) in each state, for every action in $\mathbf{\hat{A}}$, there is at
  least one action in $\mathbf{A}$ which is functionally identical in the same state, and b) in each state, for
  each action in $\mathbf{A}$ is represented by at least one action in $\mathbf{\hat{A}}$ in that state. 

{\textbf{Definition 3.1}  Given a discrete finite MDP, $M$ specified by $(\mathbf{S},\mathbf{A},p)$ and MDP, $\hat{M}$ specified by $(\mathbf{S},\mathbf{\hat{A}},\hat{p})$, $\hat{M}$ is a \textit{refinement} of $M$ when 
\[\mathop{\forall}_{\hat{a} \in \mathbf{\hat{A}}, s \in \mathbf{S}} \mathop{\exists}_{a \in \mathbf{A}} \mathop{\forall}_{s',r} \hat{p}(s',r | s,\hat{a}) = p(s',r | s,a), \text{ and } \mathop{\forall}_{a \in \mathbf{A}, s \in \mathbf{S}} \mathop{\exists}_{\hat{a} \in
\mathbf{\hat{A}}} \mathop{\forall}_{s',r} \hat{p}(s',r | s,\hat{a}) = p(s',r
| s,a),\]}
  
  Note that 
  this definition of refinement requires a \textit{state conditioned} correspondence between action behavior.
  Actions do not need to have to correspond across states.
  {
\begin{theorem}
\theoremlabel{th1}
Given a discrete finite MDP, $\hat{M}$ which is a refinement of $\hat{M}$ (Definition 3.1) then both MDPs induce the same optimal value function, \ie $\forall_s V^{*}_{\hat{M}}(s) = V^{*}_{M}(s)$.
\end{theorem}
}
We prove this by showing that optimal policies under both MDPs
induce the same value function.

\begin{lemma}
\label{equal_policy}
  For any policy $\pi_M$ on $M$, there exists a policy $\pi_{\hat{M}}$ on $\hat{M}$ such that $V^{\pi_{\hat{M}}}_{\hat{M}}(s) = V^{\pi_M}_{M}(s)$, $\forall s$, and for any policy $\pi_{\hat{M}}$ on $\hat{M}$ there exists a policy $\pi_{M}$ on $M$ such that $V^{\pi_{\hat{M}}}_{\hat{M}}(s) = V^{\pi_{M}}_{M}(s)$ $\forall s$.
\end{lemma}
{ For this lemma we introduce the notion of \textit{fundamental
actions}, which are actions which correspond to sets of actions which have the
same state and reward transition distributions in a given state. We utilize the
equivalence of fundamental actions between MDPs to construct a policy in the
new MDP which induces the same value function as a given policy in the original
MDP.} We provide proofs for \theoremref{th1} and Lemma \ref{equal_policy} in
\secref{proofs}.

\subsection{Gridworld Case Study}
\seclabel{expval}
We validate these results in a tabular grid world setting. In particular, we
measure the error in learned value functions and the induced behavior, when
conducting Q-learning with datasets with different qualities of intervening
actions. 
The agent needs to navigate from the top left of a $6 \times 6$ grid to the
bottom right with sparse reward. We generate data from a fixed, sub-optimal
policy to evaluate all methods in an offline fashion (additional details in
\secref{data}). We generate $20$K episodes with this policy, and obtain value
functions using Q-learning under the following 4 choices for the intervening
actions: 
\textbf{(1)} Ground truth actions ($V_\text{gt}$), \textbf{(2)} One
action ($V_\text{one-act}$, ammounts to TD(0) policy evaluation), \textbf{(3)}
4$\times$ refinement of original action space ($V_\text{4$\times$-gt}$). We
modify the data so that each sample for a particular action in the original
action space is randomly mapped to one of 4 actions in the augmented space.
\textbf{(4)} Obfuscated actions ($V_\text{impure-act}$). Original action with
probability $0.5$, and a random action with probability $0.5$.

\figref{tabularexp} shows the learned value functions under these different
action labels, and reports the MSE from the true value function, along with induced behavior. In line with our
expectations, 
$V_\text{4$\times$-gt}$ which uses a refinement of the actions is able to
recover the optimal value function. 
$V_\text{one-act}$ fails to recover the optimal value
function, and recovers the value corresponding to the policy that generated the
data. $V_\text{impure-act}$, under noise in action labels (non-refinement) also fails to recover the optimal value function.
Furthermore, the behavior implied by $V_\text{impure-act}$ and
$V_\text{one-act}$ is sub-optimal. 
We also analyze
the effect of the action impurity on learned values and implied
behavior. Behavior becomes increasingly inaccurate as action impurity 
increases. More details in \secref{grid-world-mse}.

\section{Latent Action Q-Learning}
\seclabel{approach}
Our analysis in \secref{theory} motivates the design of our
approach for learning behaviors from state-only experience. Our proposed
approach decouples learning into three steps: mining \textit{latent} actions
from state-only trajectories, using these latent actions for Q-learning to
obtain value functions, and learning a policy to act according to the learned
value function. As per our analysis, if learned latent actions are a {\it
state-conditioned refinement} of the original actions, Q-learning will result
in good value functions, that will lead to good behaviors. Refer to Algorithm 1
for details.

\begin{figure}
\centering
\includegraphics[width=\linewidth]{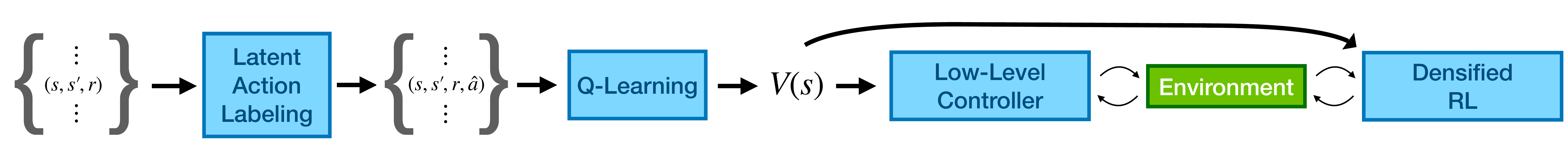}
\vspace{-5mm}
\caption{\textbf{Approach Overview.} Our proposed approach Latent Action
Q-Learning (LAQ) starts with a dataset of $(s,s',r)$ triples. Using the latent
action learning process, each sample is assigned a latent action $\hat{a}$.
Q-learning on the dataset of quadruples produces a value function, $V(s)$.
Behaviors are derived from the value function through densified RL, or
by guiding low-level controllers.}
\figlabel{overview}
\end{figure}

\subsection{Latent Actions from Future Prediction}
\seclabel{action}
Given a dataset $\mathbf{D}$ of observations streams $\ldots, o_{t}, o_{t+1},
\ldots$, the goal in this step is to learn \textit{latent} actions that are a
refinement of the actual actions that the agent executed \footnote{We use the terms state ($s_t$) and observation ($o_t$) interchangeably.}. We learn these latent
actions through future prediction. We train a future prediction model $f_\theta$, that
maps the observation $o_t$ at time $t$, and a latent action $\hat{a}$ (from a set
$\mathbf{\hat{A}}$ of discrete latent actions) to the observation $o_{t+1}$ at time
$t+1$, \ie $f_\theta(o_t, \hat{a})$. $f$ is trained to minimize a loss $l$ between the
prediction $f_\theta(o_t, \hat{a})$ and the ground truth observation $o_{t+1}$. $\hat{a}$ is
treated as a latent variable during learning. Consequently, $f_\theta$ is trained
using a form of expectation maximization~\citep{bishop2006pattern}. Each training sample
$(o_t, o_{t+1})$ is assigned to the action that leads to the lowest loss under
the current forward model. The function $f_\theta$ is optimized to minimize the
loss under the current latent action assignment. More formally, the loss
for each sample $(o_t, o_{t+1})$ is: $L(o_t, o_{t+1}) \coloneqq \min_{\hat{a} \in \mathbf{\hat{A}}} l
  \left(f_\theta(o_t, \hat{a}), o_{t+1} \right)$. We minimize $\sum_{(o_t,
  o_{t+1}) \in \mathbf{D}} L(o_t, o_{t+1})$ over the dataset to learn 
  $f_\theta$.

Latent action $\hat{a}_t$ for observation pairs $(o_{t}, o_{t+1})$ are obtained
from the learned function $f_\theta$ as: $\argmin_{\hat{a} \in
\mathbf{\hat{A}}} l \left(f_\theta(o_t, \hat{a}), o_{t+1} \right)$.
Choice of the function $f_\theta$ and loss $l$ vary depending on the problem.
We use L2 loss in the observation space (low-dimensional 
states, or images).

\subsection{Q-learning with Latent Actions}
\seclabel{q}
Latent actions mined from \secref{action} allow us to complete the given $(o_t,
o_{t+1}, r_t)$ tuples into $(o_t, \hat{a}_t, o_{t+1}, r_{t})$ quadruples for
use in Q-learning~\cite{watkins1989learning}. As our actions are discrete we
can easily adopt any of the existing Q-learning methods for discrete action
spaces (\eg \cite{mnih2013playing}). Though, we note that this Q-learning 
still needs to be done in an \textit{offline} manner from pre-recorded
state-only experience. While we adopt the most basic Q-learning in our
experiments, more sophisticated versions that are designed for offline
Q-learning (\eg \cite{KumarZTL20, fujimoto2019off}) can be directly adopted,
and should improve performance further. Value functions are obtained from the Q-functions as 
$V(s) = \max_{\hat{a} \in \mathbf{\hat{A}}} Q(s,\hat{a})$.

\subsection{Behaviors from Value Functions}
\seclabel{behavior}
Given a value function, our next goal is to derive behaviors from the learned
value function. In general, this requires access to the transition function of
the underlying MDP. Depending on what assumptions we make, this will be done in
the following two ways.

\textbf{Densified Reinforcement Learning.}
Learning a value function from state-only experience can be extremely valuable
when a dense reward function for the underlying task is not readily available.
In this case, using the learned value function can densify sparse reward
functions, making previously intractable RL problems solvable. Specifically, we
use the value function to create a \textit{potential-based} shaping function
$F(s,s') = V(s') - V(s)$, based on \cite{ng1999policy}, and construct an
augmented reward function $r'(s,a,s') = r(s,a,s') + F(s,s')$.  Our experiments
show that using this densified reward function speeds up behavior acquisition.

\textbf{Domain Specific Low-level Controllers.} 
In more specific scenarios, it
may be possible to employ hand designed low-level controllers in conjunction
with a model that can predict the next state $s'$ on executing any of low-level
controllers. In such a situation, behavior can directly be obtained by picking
the low-level controller that conveys the agent to the state $s'$ that has the
highest value under the learned $V(s)$. Such a technique was used
by~\cite{chang2020semantic}. We show results in their setup.

\section{Experiments}
\seclabel{expts}
We design experiments to assess the quality of value functions learned by LAQ
from undirected state-only experience.  We do this in two ways. First, we measure
the extent to which value functions learned with LAQ without ground truth
information agree with value functions learned with Q-learning with ground
truth action information. This provides a direct quality measure and allows us
to compare different ways of arriving at the value function: other methods in
the literature (D3G \citep{edwards2020estimating}), and simpler
alternatives of arriving at latent actions. Our second evaluation measures the
effectiveness of LAQ-learned value functions for deriving effective behavior in
different settings: when using it as a dense reward, when using it to guide
low-level controllers, and when transferring behavior across embodiments.
Where possible, we compare to behavior cloning (BC) with {\it privileged} ground
truth actions. BC with ground truth actions serves as an upper bound on the
performance of state-only imitation learning methods (BCO from \cite{TorabiWS18},
ILPO from \cite{edwards2019imitating}, \etc) and allows us to indirectly compare
with these methods.

\begin{figure}
    \centering
    \includegraphics[width=0.95\textwidth]{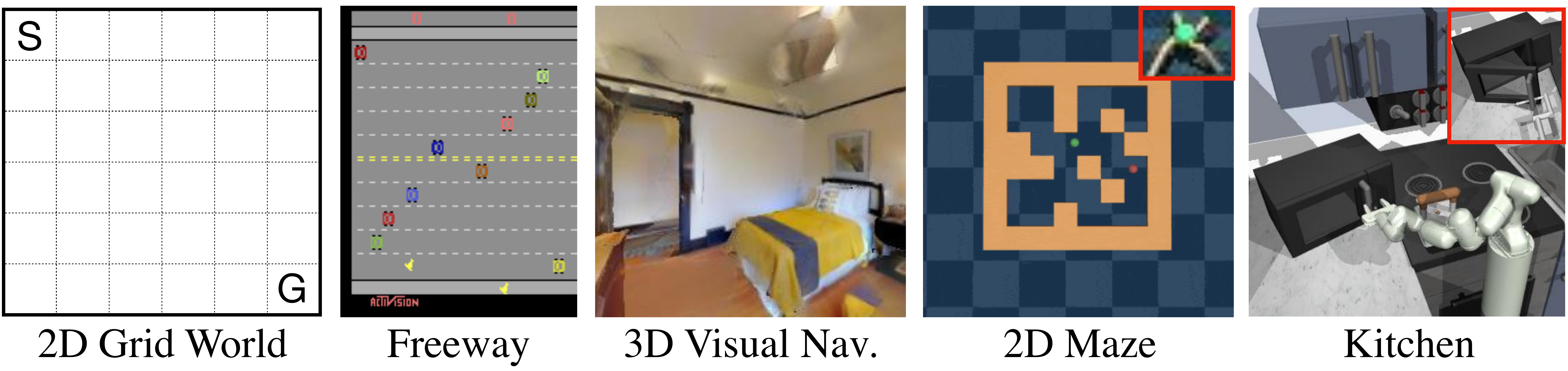}
    \caption{We experiment with five environments: 2D Grid World, Freeway (Atari), 
    3D Visual Navigation, Maze2D (2D Continuous Control), and FrankaKitchen. Top right corner of Maze2D and FrankaKitchen, shows the embodiments for cross-embodiment transfer (ant and hook, respectively).}
    \figlabel{sup_figure_envs}
\end{figure}

\textbf{Test Environments.} Our experiments are conducted in five varied
environments: the grid world environment from \secref{theory}, the Atari game
Freeway from \cite{bellemare13arcade}, 3D visual navigation in realistic
environments from \cite{chang2020semantic, habitat}, and two continuous control
tasks from \cite{fu2020d4rl}'s D4RL: Maze2D (2D continuous control navigation),
and FrankaKitchen (dexterous manipulation in a kitchen).  For Maze2D and
FrankaKitchen environments, we also consider {\it embodiment transfer}, where
we seek to learn policies for an ant and a hook respectively from the
observation-only experience of a point mass and the Franka arm.  Together,
these environments test our approach on different factors that make policy
learning hard: continuous control, high-dimensional observations and control,
complex real world appearance, 3D geometric reasoning, and learning across
embodiments. Environments are visualized in \figref{sup_figure_envs}, 
more details about the environments are provided in~\secref{env-details}.

\textbf{Experimental Setup}
For each setting, we work with a pre-collected dataset of experience in the
form of state, next state and reward triplets, $(o_t, o_{t+1}, r_t)$.  We use
our latent-variable forward model (\secref{action}) and label triplets
with latent actions to obtain quadruples $(o_t, \hat{a}_t, o_{t+1}, r)$. We
perform Q-learning on these quadruples to obtain value functions $V(s)$, 
which are used to acquire behaviors either through densified RL by
interacting with the environment, or through use of
domains-specific low-level controllers. We use the ACME
codebase~\citep{hoffman2020acme} for experiments.

\textbf{Latent Action Quality.}
In line with the theory developed in \secref{theory}, we want to establish how well our method learns a refinement of the underlying action space. To assess this, we study the {\it state-conditioned purity} of the partition induced by the learned latent actions (see definition in \secref{purity}). Overall, our method is effective at finding refinement of the original action space. It achieves higher state-conditioned purity than a single action and clustering. In high-dimensional image observation settings, it surpasses baselines by a wide margin. More details in \secref{purity}.

\renewcommand{\arraystretch}{1.1}
\begin{table*}
\centering
\small 
\setlength{\tabcolsep}{6pt}
\caption{We report Spearman's correlation coefficients for value functions learned using various methods with DQN, against a value function learned offline using ground-truth actions (DQN for discrete action environments, and DDPG for continuous action environments). The {\it Ground Truth Actions} column shows Spearman's correlation coefficients between two different runs of offline learning with ground-truth actions. See \secref{q-quality}. Details on model selection in \secref{vis_value_fn}. }
\tablelabel{spearmans}
\resizebox{\textwidth}{!}{
\begin{tabular}{lcccccc}
\toprule
\textbf{Environment}           & \textbf{D3G}                     & \textbf{Single Action} & \textbf{Clustering} & \textbf{Clustering (Diff)} & \textbf{Latent Actions} & \textbf{Ground Truth Actions}\\
\midrule
2D Grid World         & 0.959                   & 0.093 & 0.430 & \textbf{1.000} & 0.985 & 1.000\\
Freeway               & -- (image input) & 0.886 & 0.945 & 0.902 & \textbf{0.961} & 0.970\\
3D Visual Navigation  & -- (image input) & 0.641 & 0.722 & 0.827 & \textbf{0.927} & 0.991\\
\arrayrulecolor{black!30}\midrule
2D Continuous Control & 0.673                   & 0.673 & 0.613 & 0.490 & \textbf{0.844} & 0.851\\
Kitchen Manipulation  & 0.854                   & 0.858 & 0.818 & 0.815 & \textbf{0.905} & 0.901 \\
\arrayrulecolor{black} \bottomrule
\end{tabular}}
\end{table*}

\subsection{Quality of Learned Value Functions}
\seclabel{q-quality}
We evaluate the quality of the value functions learned through LAQ. 
We use as reference the value
function $V_{\text{gt-act}}$, obtained through \textit{offline} Q-learning
(DDPG for continuous action cases) with
true ground truth actions \ie $(o_t, a_t, o_{t+1}, r_t)$.\footnote{
Offline DDPG in the FrankaKitchen environment was unstable. 
To obtain a reference value function, we manually define a value function based
on the distance between the end-effector and the microwave handle (lower 
better), and the angle of the microwave door (higher better). We use this 
as the reference value function.}
For downstream
decision making, we only care about the relative ordering of state values.
Thus, we measure the Spearman's rank correlation coefficient between the
different value functions. \tableref{spearmans} reports the Spearman's
coefficients of value functions obtained using different action labels: single
action, clustering, latent actions (ours), and ground truth actions.
We also report Spearman's correlations of value functions produced using
D3G \citep{edwards2020estimating}. In all settings we do Q-learning over the
top 8 dominant actions, except for Freeway, where using the top three actions
stabilized training.
Our method out performs all baselines in settings with
high-dimensional image observations (3D Visual Navigation, Freeway). In state
based settings, where clustering state differences is a helpful inductive bias,
method is still on-par with, or superior to clustering state
differences and even D3G, which predicts state differences.

\subsection{Using Value Functions for Downstream Tasks}
\begin{figure*}
\centering
\insertWL{0.32}{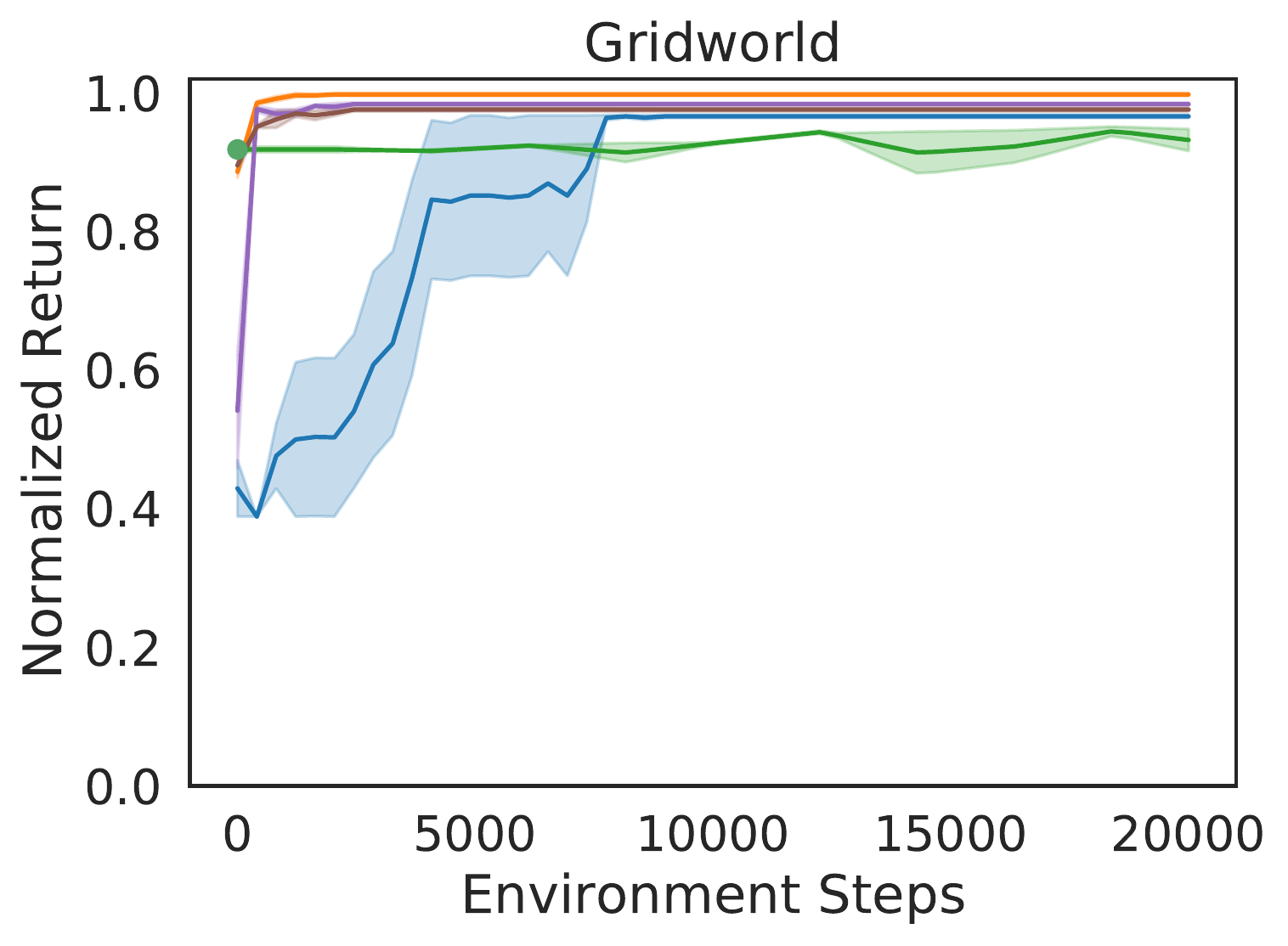} \hfill 
\insertWL{0.32}{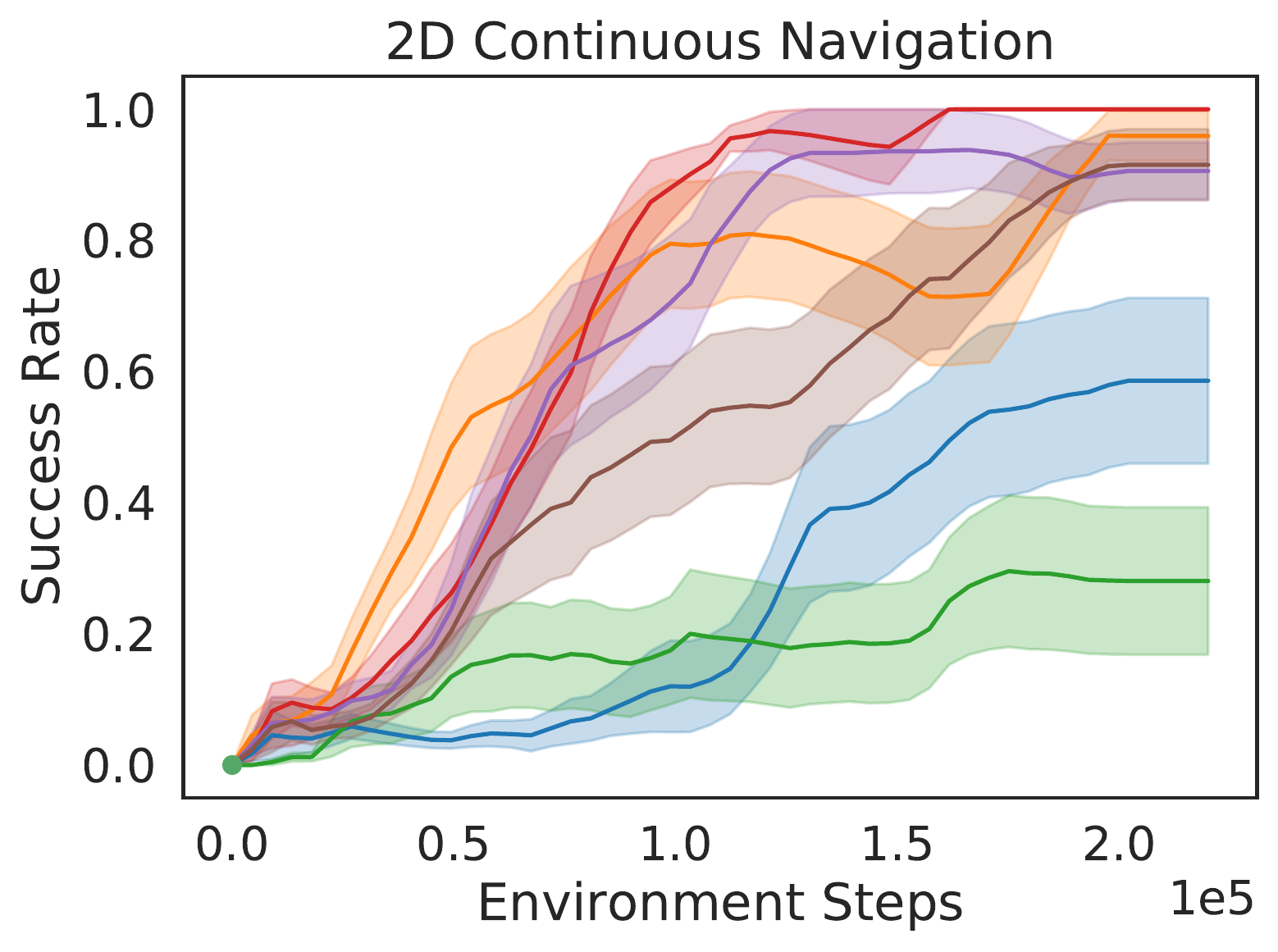} \hfill
\insertWL{0.32}{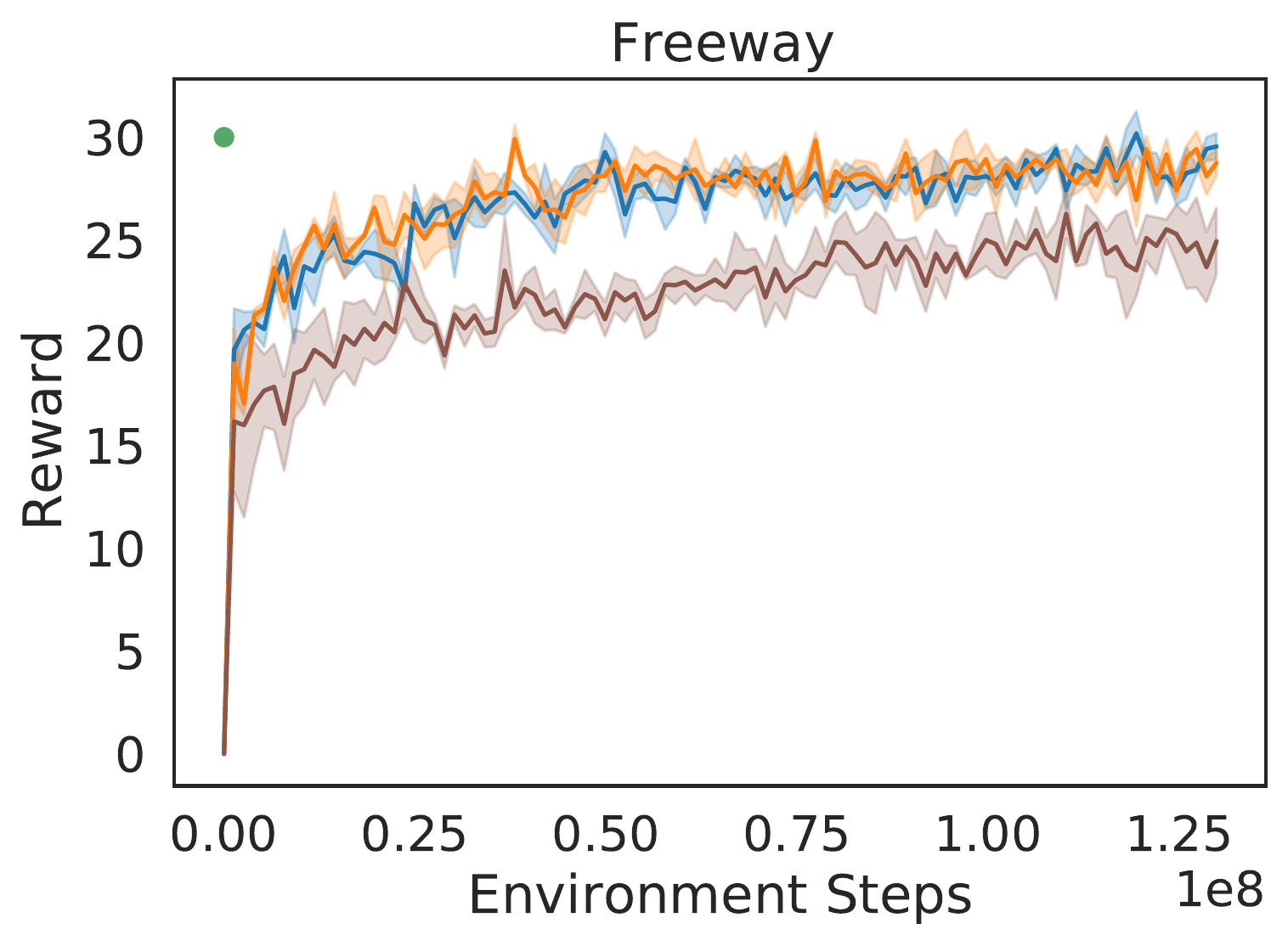} \\
\insertWL{1}{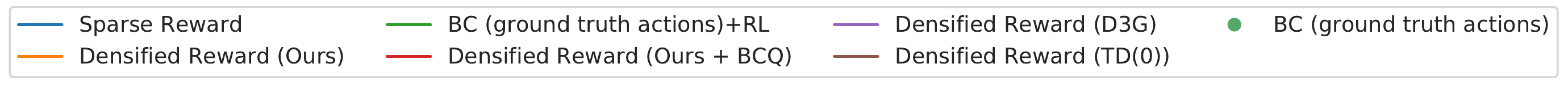}
\vspace{-6mm}
\caption{We show learning curves for acquiring
  behavior using learned value functions. We compare densified RL
  (\secref{behavior}) with sparse RL and BC/BC+RL. See \secref{rl_exps} for
  more details. Results are averaged over 5 seeds and show $\pm$
  standard error. 
}
\figlabel{behavior}
\vspace{-4mm}
\end{figure*}

\seclabel{rl_exps}
Our next experiments test the utility of LAQ-learned value functions 
for acquiring goal-driven behavior. We first describe the 3  
settings that we use to assess this, and then summarize our takeaways.

\begin{itemize}[nolistsep,leftmargin=*]
\item \textbf{Using value functions as dense reward functions}.  We combine
sparse task reward with the learned value function as a potential function (\secref{behavior}). We scale up the sparse task rewards by a factor
of 5 so that behavior is dominated by the task reward once 
policy starts solving the task. \figref{behavior} measures the learning sample
efficiency. We compare to only using the sparse reward, behavior
cloning (BC) with ground truth actions, and BC followed by spare reward RL.
\item \textbf{Using value functions to learn behavior of an agent with a
different embodiment.} 
Decoupling the learning of value function and the policy
has the advantage that learned value functions can be used to improve learning
across embodiment.  We demonstrate this, we keep the same task, but change the
embodiment of the agent in Maze2D and FrankaKitchen environments. {Note that we do 
not assume access to ground truth actions in these experiments either.} For Maze2D,
the point-mass is replaced with a 8-DOF quadrupedal ant. For FrankaKitchen, the
Franka arm is replaced with a position-controlled hook. 
We may need to define how we query the value function when the embodiment (and
the underlying state space) changes.
For the ant in Maze2D, the location (with $0$ velocity) of the ant body is used to
query the value function learned with the point-mass. For the hook in
FrankaKitchen, the value function is able to transfer directly as both settings
observe end-effector position and environment state. 
We report results in \figref{transfer}. 
\item \textbf{Using value functions to guide low-level controllers.}
Learned value functions also have the advantage that they can be used directly
at test time to guide the behavior of low-level controllers.  We do this
experiment in context of 3D visual navigation in a scan of a real building and
use the \textit{branching environment} from~\cite{chang2020semantic}. We follow
their setup and replace their value functions with ones learned using LAQ in
their hierarchical policy, and compare the efficiency of behavior encoded by
the different value functions. 
\end{itemize}

\textbf{LAQ value functions speed up downstream learning.}
Learning plots in \figref{behavior} show that LAQ-learned value functions speed
up learning in the different settings over learning simply with sparse rewards
(orange line \vs blue line).  In all settings except Freeway, our method not
only learns more quickly than sparse reward, but converges to a higher mean
performance.

\textbf{LAQ discovers stronger behavior than imitation learning when faced with
undirected experience.} An advantage of LAQ over other imitation-learning based
methods such as BCO~\citep{TorabiWS18} and ILPO~\citep{edwards2019imitating} is
LAQ's ability to learn from sub-optimal or undirected experience. To showcase
this, we compare the performance of LAQ with behavior cloning (BC) with ground
truth actions. Since BCO and ILPO recover ground truth actions to perform
behavior cloning (BC), BC with ground truth actions serves as an upper bound on
the performance of all methods in this class.
Learning plots in \figref{behavior} shows the effectiveness of LAQ over BC
and BC followed by fine-tuning with sparse rewards for environments where the
experience is undirected (Maze2D, and GridWorld). For Freeway, the experience is
fairly goal-directed, thus BC already works well. A similar trend
can be seen in the higher Spearman's coefficient for LAQ \vs \V{one-act} in
\tableref{spearmans}. LAQ discovers stronger
behavior than imitation learning when faced with undirected data.  

\textbf{LAQ is compatible with other advances in batch RL.}
Although LAQ uses the most basic Q-Learning as our offline value learning
method, it is compatible with recent more advanced offline RL value-learning
methods (such as CQL~\citep{KumarZTL20} and BCQ~\citep{fujimoto2019off}). We
validate by simply swaping to using
(discrete) BCQ with our latent actions. {\figsref{behavior}{transfer} show that LAQ with
BCQ is the strongest method, outperforming ours with DQN, and D3G, on Maze2D and embodiment transfer environments. }
Analysis of Spearman's correlations in \tableref{spearmans_bcq} shows the same
trend as before with latent actions: better than single actions,
and clustering variants. Note also that use of BCQ leads to value functions
with better Spearman's correlations than DQN.

\renewcommand{\arraystretch}{1.1}
\begin{table*}
\centering
\setlength{\tabcolsep}{6pt}
\caption{We report Spearman's correlation coefficients for value functions learned using either DQN or BCQ, against a value function learned offline using BCQ with ground-truth actions. The {\it Ground Truth Actions} column shows Spearman's correlation coefficients between two different runs of offline learning with ground-truth actions. See \secref{q-quality}.}
\tablelabel{spearmans_bcq}
\resizebox{\textwidth}{!}{
\begin{tabular}{lccccc}
\toprule
\textbf{Environment}  & \textbf{Single Action} & \textbf{Clustering} & \textbf{Clustering (Diff)} & \textbf{Latent Actions} & \textbf{Ground Truth Actions}\\
\midrule
2D Continuous Control (DQN) & 0.664 & 0.431 & 0.312 & \textbf{0.807} & 0.765\\ %
2D Continuous Control (BCQ) & 0.710 & 0.876 & 0.719 & \textbf{0.927} & 0.990 \\ %
\arrayrulecolor{black} \bottomrule
\end{tabular}}
\end{table*}

\textbf{LAQ value functions allow transfer across embodiments.}
\figref{transfer} shows learning plots of agents trained with
cross-embodiment value functions. LAQ-densified rewards functions,
speed-up learning and consistently guide to higher reward
solutions than sparse task rewards, or D3G.

\textbf{LAQ compares favorably to D3G.}
We compare LAQ and
D3G (a competing state-only method) in different ways. 
D3G relies on generating potential 
future states. This is particularly challenging for 
image observations, and D3G doesn't show results with 
image observations. In contrast, LAQ maps state transitions 
to discrete actions, and hence works with image 
observations as our experiments show. Even in scenarios
with low-dimensional state inputs, LAQ learns better
value functions than D3G, as evidences by Spearman's 
correlations in \tableref{spearmans}, and learning plots in
\figref{behavior} and \figref{transfer}.

\begin{figure}
\centering
\begin{minipage}{0.6\textwidth}
\vspace{-3mm}
\centering
 \insertH{0.56}{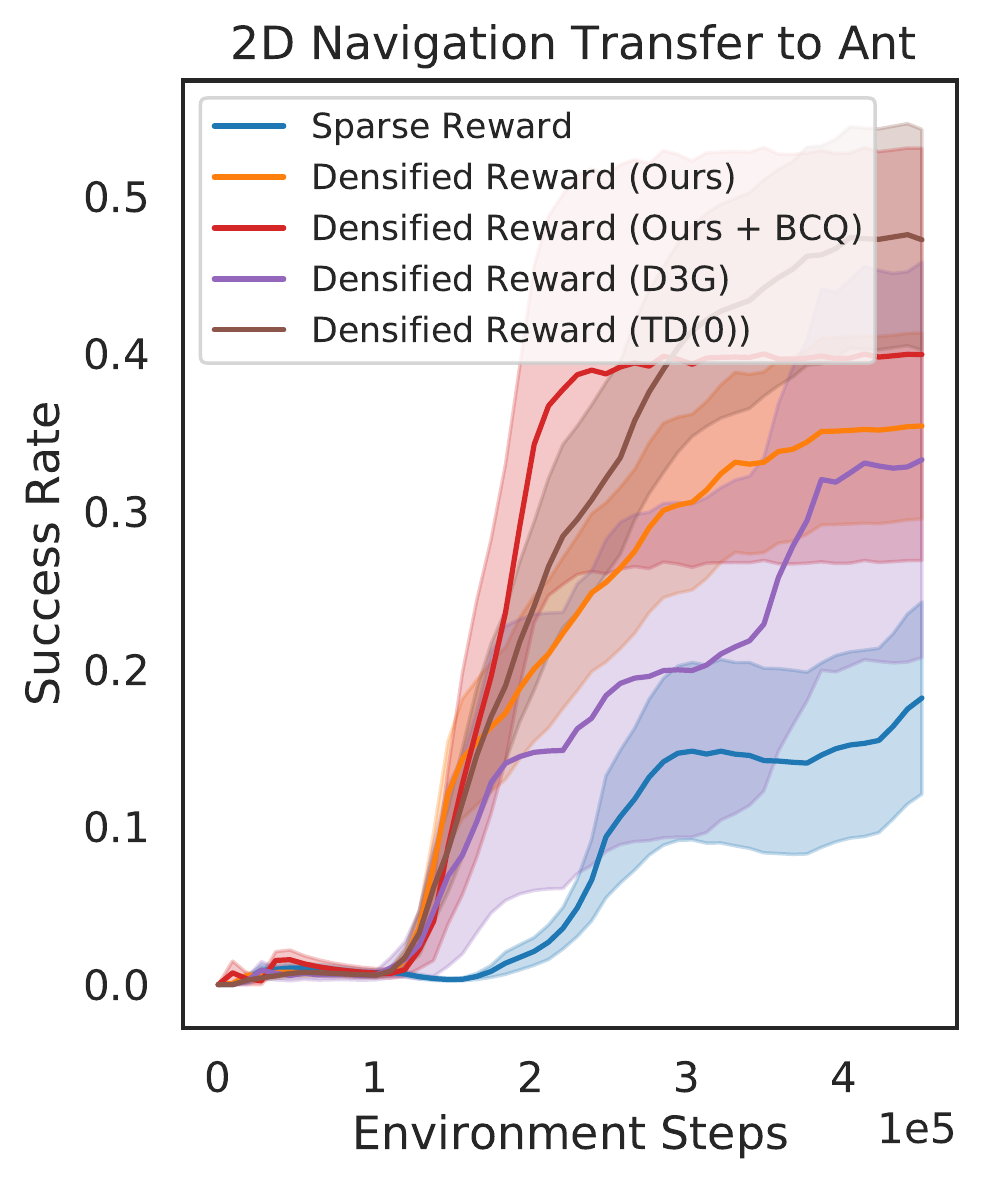} 
 \insertH{0.56}{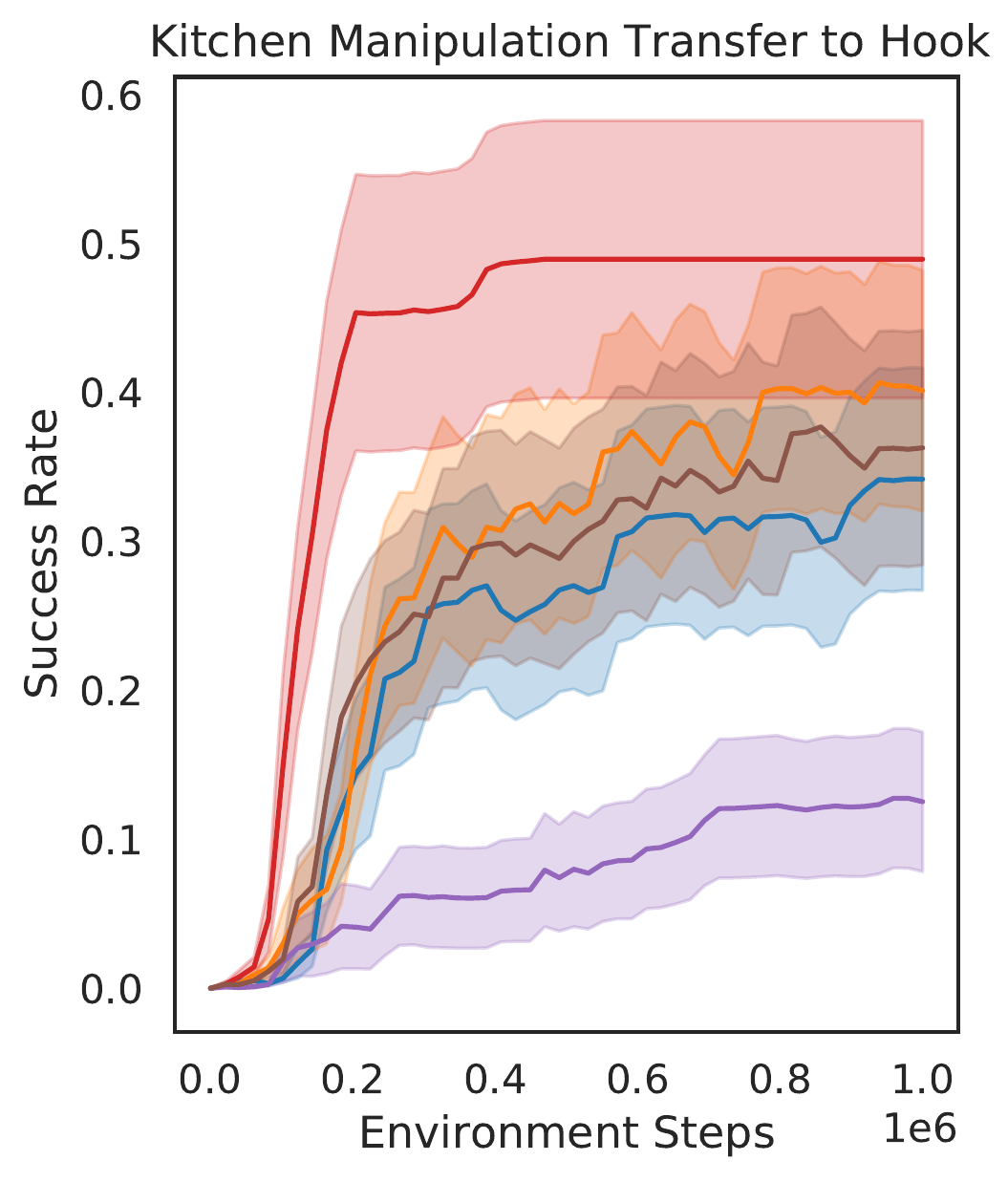}
\vspace{-1.6mm}
\caption{Behavior acquisition across embodiments. Results averaged over 50 seeds and show $\pm$ standard error.}
\figlabel{transfer}
\end{minipage}
\hspace{1mm}
\begin{minipage}{0.36\textwidth}
\setlength{\tabcolsep}{2pt}
\vspace{-1em}
\insertWL{1.00}{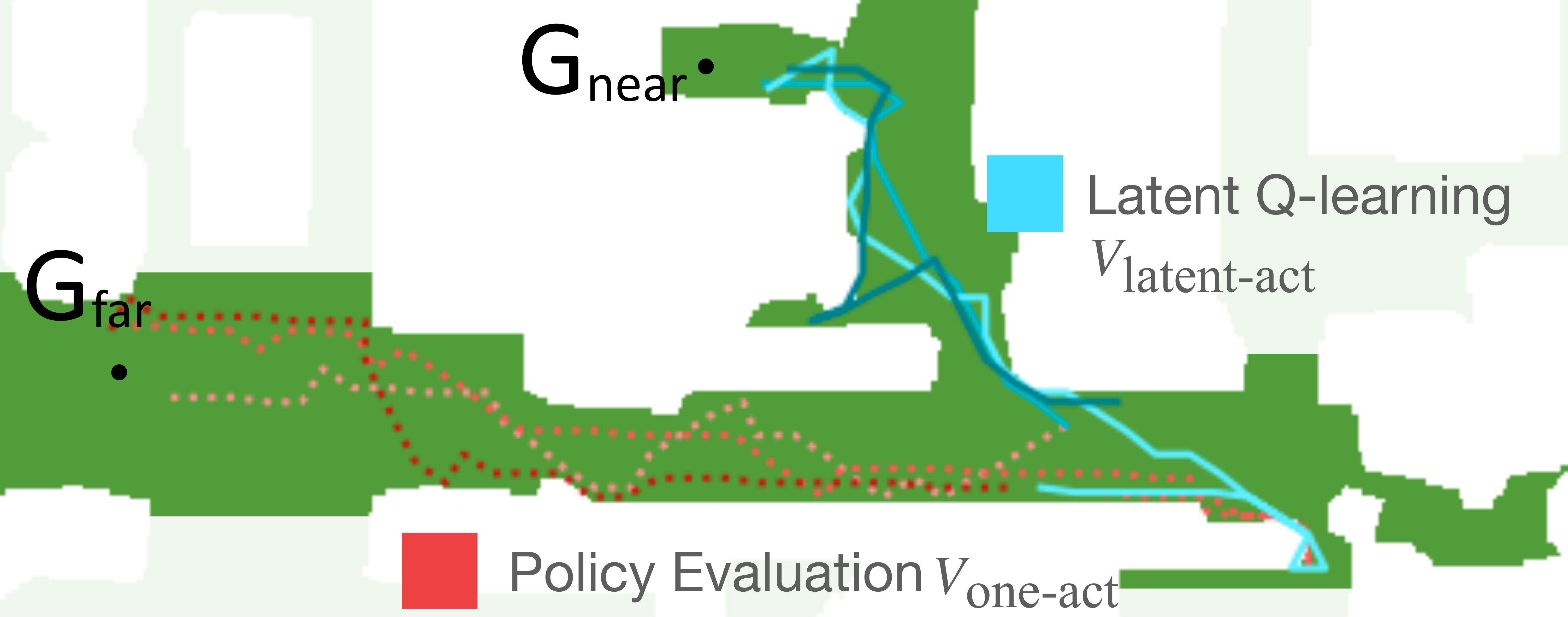}
\resizebox{1.0\linewidth}{!}{
\begin{tabular}{lcc}
\toprule
    & \textbf{Interaction} & \\ 
    & \textbf{Samples} & \textbf{SPL}\\
\midrule
$V_\text{one-act}$~\citep{chang2020semantic} & 0 & 0.53\\
$V_\text{cluster-act}$  & 0 & 0.57\\
$V_\text{latent-act}$  & 0 & 0.82\\
$V_\text{inverse-act}$~\citep{chang2020semantic} & 40K & 0.95\\
\bottomrule
\end{tabular}}
\caption{Visualization of trajectories and SPL numbers in the 3D visual navigation environment.}
\figlabel{3d-nav}
\end{minipage}

\end{figure}

\textbf{LAQ value functions can guide low-level controllers for zero-shot control:}
We report the SPL for 3D navigation using value functions combined with
low-level controllers in \figref{3d-nav}. We report the efficiency of behavior
induced by LAQ learned value functions as measured by the SPL metric from \cite{anderson2018evaluation} (higher is
better).
The branching environment has two goal states, one optimal and one sub-optimal. The demonstrations there-in were specifically
designed to emphasize the utility of knowing the intervening actions. Simple
policy evaluation leads sub-optimal behavior (SPL of 0.53) and past work relied
on using an inverse model to label actions~\citep{chang2020semantic} to derive
better behavior. This inverse model itself required $40K$ interactions with the
environment for training, and boosted the SPL to 0.95. LAQ is able to navigate to the optimal goal (w/ SPL 0.82) but without the $40K$ online interaction samples
necessary to acquire the inverse model. It also performs better than clustering
transitions, doing which achieves an SPL of 0.57. The improvement is borne out
in visualizations in \figref{3d-nav}. LAQ correctly learns to go to the
nearer goal, even when the underlying experience came from a policy that
preferred the further away goal.

\section{Discussion}

Our theoretical characterization and experiments in five representative
environments showcase the possibility and potential of deriving 
goal-directed signal from undirected state-only experience. Here we discuss 
some scenarios which are fundamentally hard, and some avenues for future
research.

\textbf{Non-deterministic MDPs.} Our theoretical result relies on
a refinement where state-action transition probabilities are matched. However,
the latent action mining procedure in LAQ results in deterministic
actions. Thus, for non-deterministic MDPs, LAQ will be unable to achieve a strict refinement, leading to 
sub-optimal value functions.
However, note that this limitation isn't specific to our method, but applies 
to {\it any} deterministic algorithm that seeks to learn from 
observation only data. We formalize this concept and provide a 
proof in \secref{stoch_mdp}.

\textbf{Constraining evaluation of $V(s)$ to within its domain.} LAQ learns a
value function $V(s)$ over the set of states that were available in the
experience dataset, and as such its estimates are only accurate within this
set. In situations where the experience dataset doesn't span the entire state
space, states encountered at test time may fall out of the distribution used to
train $V(s)$, leading to degenerate solutions. In \secref{1-norm} we discuss a
density model based solution we used for this problem, 
along with an alternate parameterization of value networks which helps avoid degenerate solutions.

\textbf{Offline RL Validation.} Validation (\eg when to stop training)
is a known issue in offline RL~\citep{gulcehre2020rl}. Like other 
offline RL methods, LAQ suffers from it too. LAQ's use of Q-learning 
makes it compatible to recent advances~\citep{kumar2021workflow}
that tackle this validation problem.

\section{Related Work}
\seclabel{related}
Our work focuses on batch (or offline) RL with state-only data using a
latent-variable future prediction model. We survey works
on batch RL, state-only learning, and future prediction.

\textbf{Batch Reinforcement Learning.}
As the field of reinforcement learning has matured, batch
RL~\citep{lange2012batch, levine2020offline} has gained attention as
a component of practical systems. A large body of work 
examines solutions the problem of extrapolation error in batch RL settings. Advances in these works are complementary to our approach, as substantiated by our experiments with BCQ. { A more detailed discussion of batch RL methods can be found in \secref{related-cont}.}

\textbf{State-only Learning.}
Past works have explored approaches for dealing with the lack of
actions in offline RL when given {\it goal-directed} or {\it undirected}
state-only experience. Works in the former category rely on high quality behavior in the data, and suffer on sub-optimal data. Past work on state-only learning from undirected experience relies on either domain knowledge or state reconstruction and only show results with low dimensional states. {See \secref{related-cont} for continued discussion.}

\textbf{Future Prediction Models.} Past work from~\cite{oh2015action,
agrawal2016learning, finn2016unsupervised} (among many others)
has focused on building action conditioned forward models in pixel and latent spaces.
Yet other work in computer vision studies video prediction
problems~\citep{xue2016visual, castrejon2019improved}. Given the uncertainty in
future prediction, these past works have pursued variational (or latent
variable) approaches to make better predictions. Our latent
variable future model is inspired from these works, but we explore its
applications in a novel context.   

{\textbf{Latent MDP Learning}
One way to interpret our method is that of learning an approximate MDP homomorphism \citep{taylor2008bounding,ravindran2004approximate}. Other works have explored learning latent homorphic MDPs. These methods tend to focus on learning equivalent latent state spaces \citep{li2006towards,givan2003equivalence}. Most similarly to our work \citep{vanderpol2020plannable} also learns a latent action space, but relies on ground truth action data to do so.}

\textbf{Acknowledgement}: This paper is based on work supported by NSF
under Grant \#IIS-2007035.

\bibliography{refs}
\bibliographystyle{iclr2022_conference}

\appendix
\renewcommand{\thetable}{S\arabic{table}}%
\renewcommand{\thefigure}{S\arabic{figure}}%
\newpage

\section{Appendix}
\setlength{\intextsep}{4pt}

\subsection{Related Work Continued}
\seclabel{related-cont}
Discussion continued from \secref{related}:

\textbf{Batch Reinforcement Learning.} In recent times, \cite{gulcehre2020rl} and
\cite{fu2020d4rl} propose datasets and experimental setups for studying offline
RL problems. A large body of work 
examines solutions to the batch RL problem. Researchers have identified that
{\it extrapolation error}, the phenomenon in which batch RL algorithms incorrectly
estimate the value of states/actions not present in the
training batch, is a major challenge, and have proposed methods to tackle it,
\eg BCQ \citep{fujimoto2019off}, BEAR \citep{Kumar19BEAR},
IRIS from \citep{mandlekar2020IRIS}, and CQL \citep{KumarZTL20} among many others.
In contrast to these model-free methods, \cite{argenson2021modelbased,
rajeswaran2019meta, rafailov2020offline} learn a forward
predictive model from the batch data and use it for model predictive control.
These methods all approach the traditional batch RL problem, while we consider
a different and harder setting in which the action labels are unavailable.
Aforementioned advances in offline RL are complementary to our work. Offline
value learning approaches (such as CQL and BCQ) can serve as a drop-in
replacement for Q-learning in our pipeline and improve our results. In fact,
our experiments with BCQ substantiate this.

\textbf{State-only Learning.}
In line of work which studies learning form {\it goal-directed} state-only experience, researchers use imitation learning-based techniques
\citep{rado2020soil, TorabiWS18, edwards2019imitating, kumar2019learning}, learn
policies that match the distribution of visited states~\citep{gail, gaifo,
torabi2019imitation}, or use demonstrations to construct dense reward functions
\citep{shao2020concept, SermanetXL17, SinghYFL19, XieSLF18,
edwards2019perceptual}. These methods make
strong assumptions about the quality and goal-directed nature of the experience
data, and suffer in performance when faced with low-quality or undirected
experience.  

Instead of goal-directed experience our work tackles the problem of learning from undirected experience. Past work in this area employs Q-learning
to learn optimal behavior from sub-optimal data~\citep{chang2020semantic,
song2020grasping, edwards2020estimating}. \cite{chang2020semantic} and
\cite{song2020grasping} use domain specific insights. 
\cite{edwards2020estimating} rely on being able to generate the next state and
only demonstrate results in environments with low-dimensional states. Instead,
our work maps transition tuples to discrete latent actions and can thus easily
work with high-dimensional observations such as \rgb images.
\newpage
\subsection{Proofs}
\seclabel{proofs}
\begin{proof}
\seclabel{proof3.2}
  \textbf{Proof for Lemma \ref{equal_policy}.} We will start by showing that for any policy $\pi_{\hat{M}}$ on $\hat{M}$, there exists a policy $\pi_{M}$ on $M$ such that $V^{\pi_{\hat{M}}}_{\hat{M}}(s) = V^{\pi_{M}}_{M}(s)$, for all $s$.\\
  To do this we need to introduce the idea of \textit{fundamental actions}, which are classes of actions which have the same state and reward transition distributions in a given state. If we have a fundamental action $b$, corresponding to some state and reward distributions, let $\alpha(b,s) \subseteq \mathbf{A}$ give the set of actions in $\mathbf{A}$ that have the matching state and reward transition distributions,
  \[\mathop\forall_{a \in \alpha(b,s)} p(s',r|s,a) = p(s',r|s,\alpha(b,s)_1).\]
  Similarly, let $\hat\alpha(b,s) \subseteq \mathbf{\hat{A}}$ give the set of actions in $\mathbf{\hat{A}}$ belonging to $b$ in state $s$.
   In any given state, there are at most $\text{min}(|\mathbf{A}|,|\mathbf{\hat{A}}|)$ fundamental actions for $M$ and the union of all actions belonging to all fundamental actions gives the set of actions that make up the original action space. For our MDP let's denote $B(s)$ as the set of fundamental actions in the state $s$ for $M$, and $\hat{B}(s)$ as the set of fundamental actions in the state $s$ for $\hat{M}$. 
   Let $\beta(s,a)$, and $\hat\beta(s,a)$ be functions which return the set of actions which correspond to the same fundamental action containing as $a$ in state $s$ for $M$, and $\hat{M}$ respectively. This means,
  \begin{equation}
  \mathop\bigcup_{b \in B(s)}\alpha(b,s) = A, 
  \mathop\forall_{b \in B(s)} \mathop\forall_{b' \ne b}\alpha(b,s) \cap
  \alpha(b',s) = \emptyset,
  \text{ and} \mathop{\forall}_{a \in A, s} \exists_{b \in B(s)} | a \in \alpha(b,s).
  \eqlabel{props}
  \end{equation}

  With this notation out of the way we can construct a policy $\pi_{\hat{M}}$ from $\pi_M$, which achieves the same value in $\hat{M}$ as $\pi_M$ does in $M$. We do this by constructing $\pi_{\hat{M}}$ such that the probability distributions over each \textit{fundamental action} is equivalent. We define this policy as
  \[
    \pi_{M}(a|s) = \frac{1}{|\beta(s,a)|}\sum_{\hat{a} \in \hat{\beta}(s,a)}\pi_{\hat{M}}(\hat{a}|s).
  \]
 Following \cite{sutton2018rli} we can define the value function for a given policy as 
\[V^{\pi_{M}}_{M}(s) = \mathbb{E}_{\pi_{M}}[G_t | S_t = s],\]
  \[V^{\pi_{M}}_{M}(s) = \sum_{a \in \mathbf{A}} \pi_{M}(a|s) \sum_{s',r}p(s',r|s,a)\big[r + \gamma V^{\pi_{M}}_{M}(s')\big].\]
  From the properties in \eqref{props},
  we know a sum over fundamental actions will count each action exactly once, so we can write this sum as 
 \begin{equation}
   \eqlabel{value_fundamental}
  V^{\pi_{M}}_{M}(s) = \sum_{b \in B(s)}\sum_{a \in \alpha(b,s)} \pi_{M}(a|s) \sum_{s',r}p(s',r|s,a)\big[r + \gamma V^{\pi_{M}}_{M}(s')\big].
 \end{equation} 
  substituting in the constructed policy we have
  \[V^{\pi_{M}}_{M}(s) = \sum_{b \in B(s)}\sum_{a \in \alpha(b,s)} 
  \bigg[\frac{1}{|\beta(s,a)|}\sum_{\hat{a} \in \hat{\beta}(s,a)}\pi_{\hat{M}}(\hat{a}|s) \bigg]
  \sum_{s',r}p(s',r|s,a)\big[r + \gamma V^{\pi_{M}}_{M}(s')\big].\]
  since the third sum is over $\hat{\beta}(s,a)$ with $a \in \alpha(b,s)$, we can rewrite this sum as over $\hat{a} \in \hat{\alpha}(b,s)$, giving
  \[V^{\pi_{M}}_{M}(s) = \sum_{b \in B(s)}\sum_{a \in \alpha(b,s)} 
  \bigg[\frac{1}{|\beta(s,a)|}\sum_{\hat{a} \in \hat{\alpha}(b,s)}\pi_{\hat{M}}(\hat{a}|s) \bigg]
  \sum_{s',r}p(s',r|s,a)\big[r + \gamma V^{\pi_{M}}_{M}(s')\big].\]
  For the final sum, from the definition of fundamental actions, we know $p(s',r|s,a) = p(s',r|s,\alpha(b,s)_1) = \hat{p}(s',r|s,\hat\alpha(b,s)_1)$, so we can rewrite that term as 
  \[\sum_{s',r}\hat{p}(s',r|s,\hat\alpha(b,s)_1)\big[r + \gamma V^{\pi_{M}}_{M}(s')\big].\]
  Crucially, in the second sum (over $a \in \alpha(b,s)$), $\beta(s,a)$ is the same for all $a \in \alpha(b,s)$, so the term in brackets can be treated as a constant. Similarly, since $|\beta(s,a)| = |\alpha(b,s)|$ we can simplify the second sum leaving
  \[V^{\pi_{M}}_{M}(s) = \sum_{b \in B(s)}
  \sum_{\hat{a} \in \hat{\alpha}(b,s)}\pi_{\hat{M}}(\hat{a}|s)
  \sum_{s',r}\hat{p}(s',r|s,\hat\alpha(b,s)_1)\big[r + \gamma V^{\pi_{M}}_{M}(s')\big].\]
  This gives the definition of $V^{\pi_{\hat{M}}}_{\hat{M}}(s)$ using the same decomposition as equation \eqref{value_fundamental}, meaning \[V^{\pi_{M}}_M(s) = V^{\pi_{\hat{M}}}_{\hat{M}}(s).\]
One can show the opposite direction, that a policy there exists a policy $\pi_{M'}$ for any policy $\pi_M$ such $V^{\pi_{M}}_M(s) = V^{\pi_{\hat{M}}}_{\hat{M}}(s),$ with a symmetric construction.

\end{proof}

{Note that Lemma \ref{equal_policy} holds true regardless of the stochasticity of the policy, as the constructed policy matches probability of the original policy for taking each fundamental action.}\newline

\seclabel{proof3.1}
\textbf{Proof for \theoremref{th1}.} Under the new MDP $\hat{M}$, Q-learning
will learn the same optimal value function value function as learned under MDP
$M$, \ie $\forall_s V^{*}_{\hat{M}}(s) = V^{*}_{M}(s)$.
\begin{proof}
This follows from Lemma \ref{equal_policy} by contradiction. Assume there is a
state $s'$ where $V^{*}_{\hat{M}}(s) \ne V^*_{M}(s)$.  Two cases must be considered.
In the first case,
$V^{*}_{\hat{M}}(s) > V^{*}_{M}(s)$. We will notate the optimal
  policy in $\hat{M}$ as $\tilde{\pi}_{\hat{M}}^*$ (\ie $V^{*}_{\hat{M}} = V^{\tilde{\pi}_{\hat{M}}^{*}}_{\hat{M}}$
) and the optimal policy in $M$ as $\pi^*_M$.
We know there must exist a policy $\tilde{\pi}^{*}_M$ such that
  $V^{\tilde{\pi}^{*}_M}_{M}(s) = V^{\tilde{\pi}_{\hat{M}}^{*}}_{\hat{M}}(s)$ from Lemma
\ref{equal_policy}. We arrive at a contradiction because
  $V^{\tilde{\pi}^{*}_M}_{M}(s) > V^{\pi_{M}^{*}}_{M}(s)$, so $\pi_M^*$ could not have
been the optimal policy on $M$.
The other direction follows a symmetric argument. If
  $V^{\tilde{\pi}_{\hat{M}}^{*}}_{\hat{M}}(s) < V^{\pi^{*}_M}_{M}(s)$, we know there must be a
policy $\pi^{*}_{\hat{M}}$ in $\hat{M}$ such that $V^{\pi^{*}_M}_{M}(s) =
V^{\pi^*_{\hat{M}}}_{\hat{M}}(s)$. This implies $V^{\pi^*_{\hat{M}}}_{\hat{M}}(s) >
  V^{\tilde{\pi}_{\hat{M}}^*}_{\hat{M}}(s)$, meaning that $\tilde{\pi}_{\hat{M}}^*$ could not have been
the optimal policy.  
Since Q-learning is known to converge to the optimal value function \cite{watkins1989learning}, Q-learning on $\hat{M}$ will converge to the original value function $V^*_M$.
\end{proof}

\newpage
\subsection{State Only Learning From Stochastic MDPs}
\seclabel{stoch_mdp}
Here we will briefly prove that problem of deducing the optimal value function
of an MDP $M$, from some state-only experience dataset $D$, cannot be
solved in general for non-deterministic MDPs. We notate $\mathbb{D}$ as the
space of all \textit{complete} state-only datasets, and $\mathbb{V}$ as the space
of all value functions. By \textit{complete} state-only dataset, we mean
datasets in which all possible transition triples $(s,s',r)$ appear. Note that exactly one complete dataset
exists for each MDP. We denote $D_{M}$ as the complete dataset corresponding to an MDM $M$.

In short, the idea is to construct a single
state-only dataset which ambiguously could have been produced from two
different MDPs which have different optimal value functions. Since no function
deterministic function can model two outputs from the same input, no such
function can exist.

\begin{theorem}
\theoremlabel{stoch_proof}
For any function $f: \mathbb{D} \rightarrow \mathbb{V}$ which maps from the set
of of state-only datasets, $\mathbb{D}$, to a set of value functions. There
exists some MDP $M$, with corresponding dataset $D_{M}$ such that
$f(D) \ne V^*_M$.
\[\nexists f \mid \forall_M f(D_{M}) = V^*_M.\]
\end{theorem}
\begin{proof}
We proceed by contradiction, assume any $f: \mathbb{D} \rightarrow \mathbb{V}$ such that 
\[\forall_M f(D_M) = V^*_M.\]
If we can construct $M_1$ and $M_2$ such that $V^*_{M_1} \ne V^*_{2}$, and
$D$ such that $D = D_{M_1} = D_{M_2}$, then $f$ cannot produce
the correct value function for $f(D)$ in all cases, from the definition of a
function.  We will now show a very simple construction which exhibits this
property. Consider an MDP $M_1$, with 3 states, ${s_1,s_2,s_t}$, and 2 actions
$a_1,a_2$. $s_t$ is terminal, and gives reward 1, other states give reward $0$.
The transition dynamics of the two actions are given below. 

\begin{minipage}[c]{0.5\textwidth}
\centering
\captionof*{table}{\textbf{Table:} $a_1$ transitions for $M_1$}
\begin{tabular}{ |c|c|c|c| } 
 \hline
       & $s_1$ &$s_2$ & $s_t$ \\ \hline
 $s_1$ & 0     &  1    & 0 \\ \hline
 $s_2$ & 0 &   0 & 1 \\
 \hline
\end{tabular}
\end{minipage}
\begin{minipage}[c]{0.5\textwidth}
\centering
\captionof*{table}{\textbf{Table:} $a_2$ transitions for $M_1$}
\begin{tabular}{ |c|c|c|c| } 
 \hline
       & $s_1$ &$s_2$ & $s_t$ \\ \hline
 $s_1$ & 0     &  0    & 1 \\ \hline
 $s_2$ & 1 &   0 & 0 \\
 \hline
\end{tabular}
\end{minipage}
For simplicity, assume $M_1$ has a discount factor of $0.9$. This means $V^*(s_1) = V^*(s_2) = 1$.

We define $M_2$ as identical to $M_1$, only with stochastic transitions.\\
\begin{minipage}[c]{0.5\textwidth}
\centering
\captionof*{table}{\textbf{Table:} $a_1$ transitions for $M_2$}
\begin{tabular}{ |c|c|c|c| } 
 \hline
       & $s_1$ &$s_2$ & $s_t$ \\ \hline
 $s_1$ & 0     &  1    & 0 \\ \hline
 $s_2$ & 0 &   0 & 1 \\
 \hline
\end{tabular}
\end{minipage}
\begin{minipage}[c]{0.5\textwidth}
\centering
\captionof*{table}{\textbf{Table:}  $a_2$ transitions for $M_2$}
\begin{tabular}{ |c|c|c|c| } 
 \hline
       & $s_1$ &$s_2$ & $s_t$ \\ \hline
 $s_1$ &      &  0.9    & 0.1 \\ \hline
 $s_2$ & 1 &   0 & 0 \\
 \hline
\end{tabular}
\end{minipage}
In this MDP, with a discount factor of 0.9, $V^*(s_2) = 1$, but $V^*(s_1) = 0.1 + 0.9*0.9 = 0.91$.

Since both $M_1$ and $M_2$ share the same states, actions, and possible
transitions, $D_{M_1} = D_{M_2}$. So we have satisfied our condition that $D =
D_{M_1} = D_{M_2}$, and $V^*_{M_1} \ne V^*{M_2}$. Thus, no $f$ can model the
value functions for both MDPs in general.
\end{proof}

\newpage
\subsection{Grid World Behavior \vs Purity}
\seclabel{grid-world-mse}
\figref{plot} plots the MSE in value function and the proportion of
states with correct implied actions as a function of the noise in
the action labels. 
\begin{figure}[!h]
    \centering
    \includegraphics[width=\textwidth]{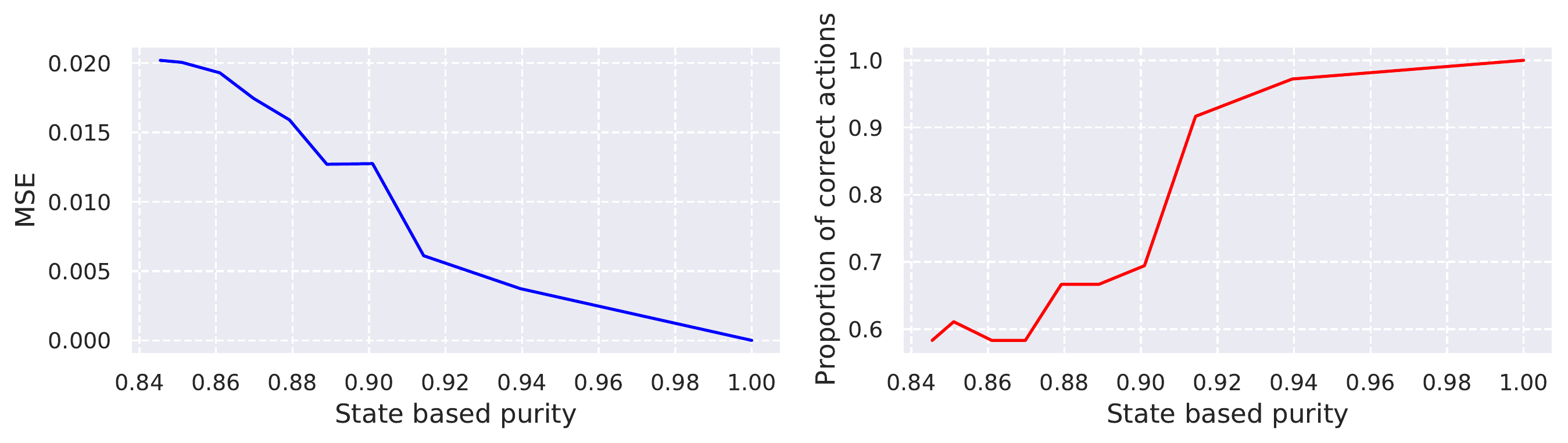}
    \caption{We plot correctness of value function estimate and behavior as a
    function of state-based purity of the intervening actions used for
    Q-learning. Left plot shows the mean squared error (lower is better)
    between the obtained value function and the optimal value function. Right
    plot shows fraction of states in which the learned value function induces
    the optimal action (higher is better). Both value function and behaviors
    become worse as state-based purity decreases.
    }
    \vspace{-5mm}
    \figlabel{plot}
\end{figure}

\newpage
\subsection{LAQ in Stochastic MDPs}
\seclabel{grid-world-stochasticity}

{We have conducted experiments on stochastic MDPs in the gridworld environment. Specifically, we do this by adding sticky actions, such that with a certain probability (called the stickiness parameter), the environment executes the previous action as opposed to the given action. We run LAQ with data from this stochastic environment and examine a) purity of latent actions, b) quality of learned value functions, and c) correctness of implied behavior.

\begin{figure}[!h]
    \centering
    \includegraphics[width=\textwidth]{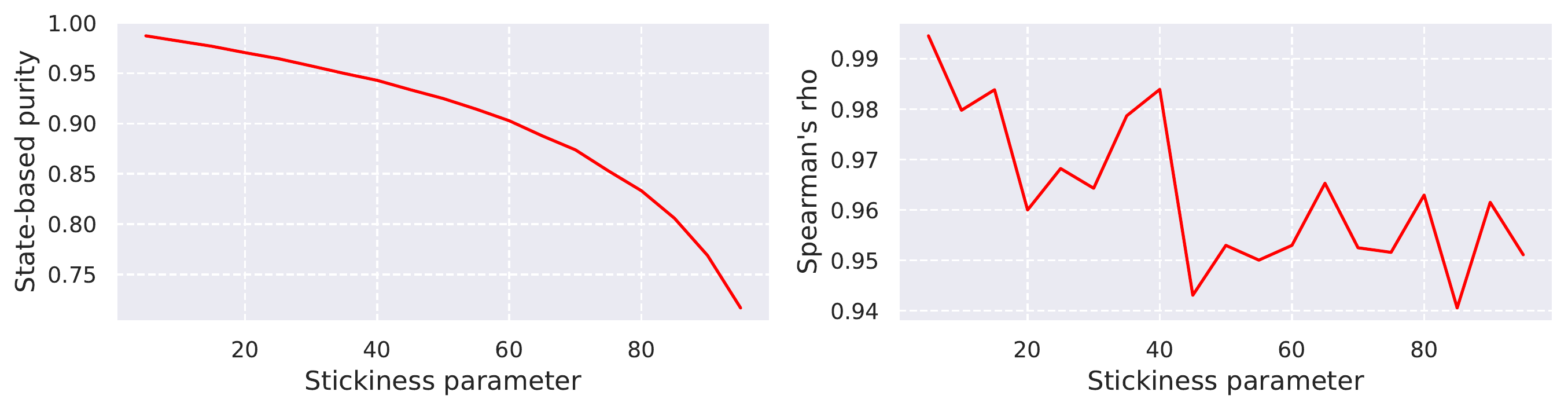}
    \includegraphics[width=\textwidth]{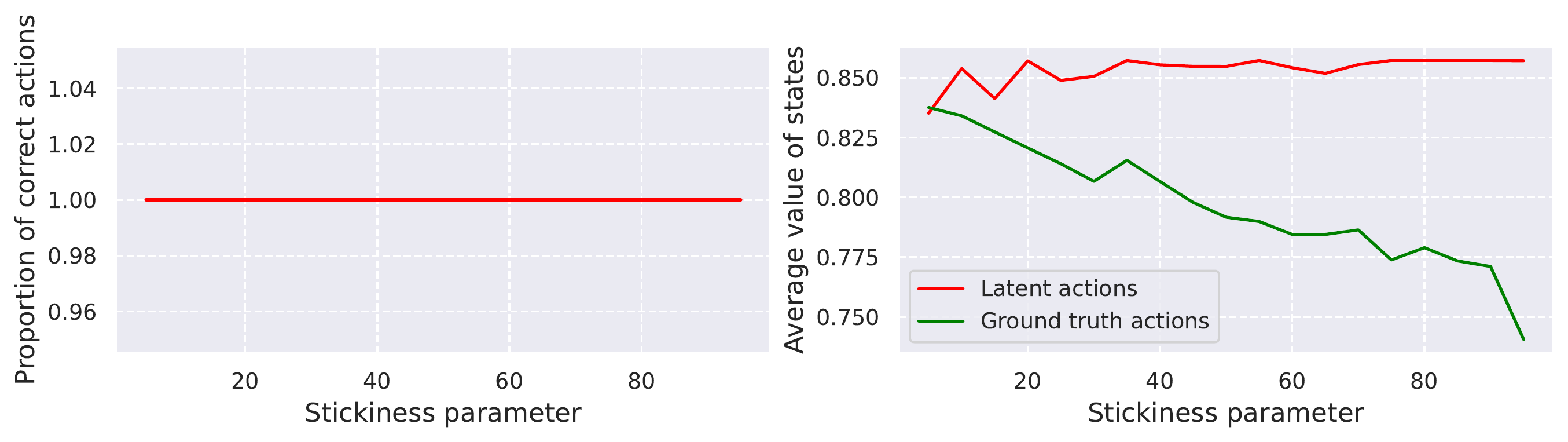}
    \caption{{As a function of stochasticity in environment, we plot:
    state-based purity of the latent actions (top left), Spearman's rank correlation between LAQ learned value function and ground truth value function (top right),
    correctness of implied behavior by the LAQ learned value function (bottom right), and
    average state-value of LAQ learend value function, and the ground truth value function (bottom left).}}
    \figlabel{stochasticity1}
\end{figure}

As expected, \figref{stochasticity1} (left) shows that the state-based purity of the learned latent actions falls off as the stochasticity (stickiness) increases. However, these impure latent actions have little effect on the Spearman's rank correlation (top right). The Spearman's rank correlation does decrease but remains reasonably high even when the stickiness parameter is set to 95\%. While the ordering of the values of the different states doesn't change by much, we note increasing amount of over-estimation in the values (bottom right). This is expected as with increasing stochasticity, there is an increasingly large gap between what LAQ thinks it can control, and what it can actually control. However, the behavior implied by LAQ is still always correct (bottom left), as the over-estimation is uniform over all states. Thus, it is possible to get good performance from LAQ in some stochastic environments. However, it is possible to also construct simple scenarios where LAQ, and for that matter any deterministic algorithm, will suffer as we describe in \secref{stoch_mdp}.}

\newpage
\subsection{Environment Details}
\seclabel{env-details}
\textbf{2D Grid World.} We use the environment and data from~\secref{expval}.
We use the $(x,y)$ coordinates as our state for this environment. We use a
multi-layer perceptron for $f_\theta$ and L2 loss for $l$.

\textbf{2D Continuous Control Navigation.} We use the Maze2D data from
D4RL~\cite{fu2020d4rl}. Observation space here is $(x,y, v_x, v_y)$ \ie location and
velocity of the agent, and action space is force along $x$ and $y$ directions.
We use a frame skip of 3. We use a multi-layer perceptron for $f_\theta$ and L2
loss for $l$. In the embodiment transfer variant, we swap the agent for a
4-legged ant, also from D4RL. The ant embodiment is far more challenging to
control, with an 8-d action space and 29-d observation space.

\textbf{Atari Game.} We work with Freeway. We generated our own data using the
protocol described in~\cite{agarwal2020optimistic}. We turned off sticky
actions and store frames both at the default resolution ($84 \times 84$) and
full resolution ($224 \times 224$).  Other settings are same as is typical:
stack last 4 frames to represent the observations, and use frame-skip of 4.  We
use a convolutional encoder-decoder model for $f_\theta$ and predict raw future
observations. We use L2 distance in the pixel space as our loss function $l$.

\textbf{3D Visual Navigation.} We work with the branching environment from
\cite{chang2020semantic} in the AI Habitat Simulator~\cite{habitat}, and use
the provided dataset of first-person trajectories as our pre-recorded
experience dataset. Following \cite{chang2020semantic}, we employ low-level
controllers for deriving behaviors from the learned value functions. We use a
convolutional encoder-decoder for $f_\theta$, and L2 loss in pixel space as $l$.

\textbf{Kitchen Manipulation.} In this task a 9-DOF Franka arm can interact
with several different objects physically simulated kitchen. The observations
are 24-d, containing the end-effector position and state of various objects in
the environment (microwave door angle, kettle position, etc.). The action space
represents the 9-d joint velocity. In cross-embodiment experiments we replace
the Franka arm with a hook. The observation space remains the same while the
action space becomes 3-d position control.

For experiments in kitchen manipulation, we utilize the GMM based solution described in \secref{1-norm}. 
For all environments, we throw out the provided action labels when learning our
models.

\newpage
\subsection{Data Collection}
\seclabel{data}

\textbf{2D Grid World.} The setting for this experiment is a $6~\times~6$ grid with 8 actions, corresponding to moving in the 4 cardinal directions and 4 diagonal directions. The agent starts in the top left (0,0), and gets reward 0 everywhere, except for in the bottom right, (5,5). Reaching the bottom right gives the agent reward 1 and terminates the episode. In the starting square (0,0), the agent has probability 0.5 of moving right, and probability 0.5 of moving down. When the agent is on the top or bottom edge of the maze (row~=~0 or row~=~5) it moves right with probability 0.9, and takes a random action with probability 0.1. When it is on the left or right edge of the grid (column = 0 or column = 5), it moves down with probability 0.9 and takes a random action with probability 0.1. Otherwise, it is in the interior of the grid, where it takes an action away from the goal (randomly chosen from: up, left, or up-left) with probability 0.9, and one of the remaining actions otherwise (randomly chosen from: down, right, down-right, up-right, down-left). This policy is rolled out for 20,000 policies to generate the data for methods described in main paper Section 3.2 and Figure~1.

\textbf{Atari Game.} 
Data for the Freeway environment had to be additionally collected, as native high resolution resolution images of the episodes from \cite{agarwal2020optimistic} are not readily available. For this, we re-generated the data using the protocal described in~\cite{agarwal2020optimistic}, without sticky actions. Then, the high resolution data was collected by taking the action sequences from the re-generated dataset and executing them in an emulator, while storing the native resolution images. Because we re-generated the data without sticky actions, the high resolution episodes are perfect recreations of the low resolution episodes.

\textbf{Visual Navigation.} 
The data used for the visual navigation experiment was generated using the same protocol as \citet{chang2020semantic}. Agents are tasked to reach one of two goals ($G_{\text{near}}$ and $G_{\text{far}}$) in a visually realistic simulator \cite{habitat}. Navigating to $G_{\text{near}}$ is the optimal path as it is nearer to the agent's starting location.

The dataset contains 3 types of trajectories, $T_1$, $T_2$, and $T_3$, representing 50\%, 49.5\%, and 0.5\% of the trajectories in the dataset respectively . $T_1$ takes a sub-optimal path to $G_{\text{near}}$, $T_3$ takes the optimal path to $G_{\text{near}}$, and $T_2$ navigates to $G_{\text{far}}$. 

\textbf{2D Continuous Control (Maze2D)} We use the \textit{Maze2D} dataset from D4RL as is.

\textbf{Kitchen Manipulation.} 
The data used for the Kitchen Manipulation environment comes from the \textit{partial} version of the D4RL FrankaKitchen dataset. To facilitate cross-embodiment transfer we convert the state representation to contain end-effector location instead of joint angles.
\newpage
\subsection{Algorithm}
\seclabel{algorithm}

\begin{algorithm}
  \caption{LAQ}
  \begin{algorithmic}[1]
    \State{Given dataset $D$ of $(o_t,o_{t+1},r)$ triples}
    \For {each epoch}
    \Comment{Latent Action Mining}
    \For {sampled batch $(o_{t}, o_{t+1}, r) \sim D$}
    \State $L(o_{t},o_{t+1}) \leftarrow  \min_{\hat{a} \in \bm{\hat{A}}} \ l(f_{\theta}(s,\hat{a}),s')$
    \State $\theta \leftarrow \theta - \alpha\nabla_{\theta}L(o_{t},o_{t+1})$ 
    \Comment{\secref{action}}
    \EndFor
    \EndFor
    \State {\textbf{let} $\hat{g}(o_t,o_{t+1}) = \argmin_{\hat{a} \in \bm{\hat{A}}} l(f_{\theta}(o_t,\hat{a}),o_{t+1})$}
    \State {$\hat{D} = \{(o_t,o_{t+1},r,\hat{g}(o_t,o_{t+1})) \mid (o_t,o_{t+1},r) \in D \}$}
    \Comment{Latent Action Labeling}
    \For {each epoch}
    \For {sampled batch $(o_{t}, o_{t+1}, r,\hat{a}) \sim \hat{D}$}
    \Comment{Q-learning with Latent Actions}
    \State Q-learning update to learn $Q(s, \hat{a})$
    \EndFor
    \EndFor
    \State $V(s) = \max_{\hat{a} \in \bm{\hat{A}}} \ Q(s, \hat{a})$
  \end{algorithmic}  
\end{algorithm} 

\newpage
\subsection{Latent Action Quality} 
\seclabel{purity}
We first measure the effectiveness of our latent action mining process by
judging the extent to which the induced latent actions are a refinement of the
original actions. We measure this using the {\it state-conditioned purity} of
the partition induced by the learned latent actions. 

{In a given state, for
any latent action $\hat{a}$, there must be some ground truth action which most
frequently mapped to $\hat{a}$. We define the purity of $\hat{a}$ as the
proportion of the most frequent action among all actions mapped to $\hat{a}$.
For example, in a given state $s$ if a set of actions $[0,0,0,1,2]$, were
mapped to latent action $\hat{a}$, then $0$ is the most frequent action mapped
to $\hat{a}$ and thus the purity of $\hat{a}$ would be $0.6$. For a given
state, the purity of an entire set of latent actions is the weighted mean of
purity of individual latent actions. Overall purity is the average of
all state wise purities weighted by how often the states appear in the dataset.}

We extend this definition of state-conditioned purity to continuous or
high-dimensional states by measuring the validation accuracy of a function $g$
that is trained to map the high-dimensional state (or observation) $s$, and the
associated latent action $\hat{a}$ to the actual ground truth action $a$.
{Training such a function induces an implicit partition of the state space.
Learning to predict the ground truth action from the latent action $\hat{a}$
within this induced partition estimates the most frequent ground truth action,
and accuracy measures its proportion, \ie purity. This exact procedure reduces
to the above definition for discrete state spaces, but also handles continuous
and high-dimensional states well. For continuous action spaces, we measure the
mean squared error in action prediction instead of classification accuracy.}

\tableref{action} reports the purity (and MSE for continuous action
environment) obtained by our proposed future prediction method. For reference,
we also report the purity obtained when using \textit{single action} that maps
all samples into a single cluster, and 2 \textit{clustering} methods that
cluster either concatenation of the two observations \ie $[o_{t+1}; o_{t}]$, or
the difference between the two observations \ie $[o_{t+1} - o_{t}]$. 
We use 8 latent actions for all environments except the FrankaKitchen
environment for which we use 64 because of its richer underlying action space.

In general, our forward models are effective at generating a latent action space
that is a state-conditioned refinement of the original action space. This is
indicated by the improvement in state-conditioned purity values over using a
single action or naive clustering.  For the 2D Grid World and 2D Continuous
Navigation, clustering in the correct space (state difference \vs state
concatenation) works well as expected. But our future prediction model, which
directly predicts $s_{t+1}$ and doesn't use any domain specific choices, is
able to outperform the corresponding clustering method. We also observe large
improvements over all baselines for the challenging case of environments with
high-dimensional state representations: Freeway and 3D Visual Navigation.

\renewcommand{\arraystretch}{1.1}
\begin{table*}[!h]
\centering
\small 
\setlength{\tabcolsep}{6pt}
\caption{We report the state-conditioned action purity (higher is better, MSE
for continuous action case where lower is better), of latent actions for
different approaches: single action, clustering concatenated observations,
clustering difference in observations, and the proposed future prediction
models from \secref{action}. We note the utility of the future prediction
model for the challenging case of Freeway and 3D Visual Navigation
environments. See \secref{purity} for a full discussion.}
\tablelabel{action}
\resizebox{\textwidth}{!}{
\begin{tabular}{lcccccccc}
\toprule
\textbf{Environment}                & \textbf{Observation}   & \textbf{Action} & \textbf{Purity}     & \textbf{Single}    & \textbf{Clustering} & \textbf{Clustering} & \textbf{Future} \\
                                    & \textbf{Space}         & \textbf{Space}  & \textbf{Metric}     & \textbf{Action}    & $[o_t, o_{t+1}]$    & $[o_{t+1}-o_{t}]$   & \textbf{Prediction} \\
\midrule
2D Grid World                       & $xy$ location          & Discrete, 8     & Purity ($\uparrow$) & 0.827              & 0.851               & {\bf 1.000}         & {0.998}      \\
Freeway                             & $210 \times 160$ image & Discrete, 3     & Purity ($\uparrow$) & 0.753              & 0.778               & 0.773               & {\bf 0.907}  \\
3D Visual Navigation (Branching)    & $224 \times 224$ image & Discrete, 3     & Purity ($\uparrow$) & 0.783              & 0.839               & 0.859               & {\bf 0.928}  \\
\arrayrulecolor{black!30}\midrule
2D Continuous Control               & $xy$ loc. \& vel      & Continuous, 2       & MSE ($\downarrow$) & 2.207               & 2.188               & {\bf 0.325}            & 0.905  \\
Kitchen Manipulation                & 24-d State             & Continuous, 8       & MSE ($\downarrow$) & 0.015               & 0.015               & 0.015            & {\bf 0.014}  \\
\arrayrulecolor{black} \bottomrule
\end{tabular}}
\vspace{-5mm}
\end{table*}

\newpage
\subsection{Analysis of Learned Latent Actions}
\seclabel{model}

We analyze the latent actions learned for the Freeway environment.

We visualize our future prediction model's predictions for the Freeway environment in \figref{reconstruction}. In line with our expectations, the one action future prediction model and the latent action future prediction model are both able to reconstruct the background and the vehicles perfectly. At the same time, the one action future prediction model fails to reconstruct the agent accurately, whereas the latent action future prediction model is able to reconstruct the agent almost perfectly. This provides evidence for the effectiveness of our proposed latent action mining approach at discovering pure action groundings.

\begin{figure*}[h]
    \centering
    \includegraphics[width=\textwidth]{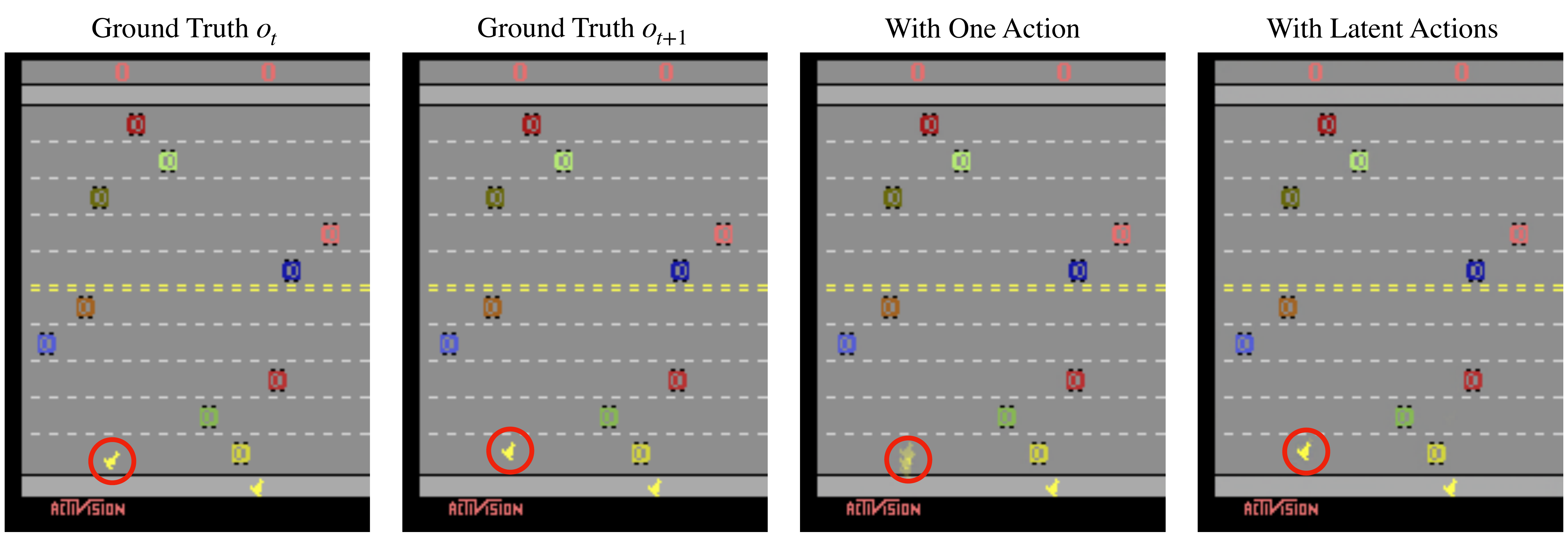}
    \caption{We visualize the Freeway future prediction models' reconstructions for $o_{t+1}$. 
    From left to right, the ground truth $o_t$, the ground truth $o_{t+1}$, the reconstruction for $o_{t+1}$ 
    by the future prediction model with one action, the reconstruction for $o_{t+1}$ 
    by the future prediction model with latent actions. The agent is circled in each image.
    }
    \figlabel{reconstruction}
\end{figure*}

\vspace{10pt}

\begin{wrapfigure}[17]{r}{0.5\textwidth}
  \begin{center}
    \includegraphics[width=0.48\textwidth]{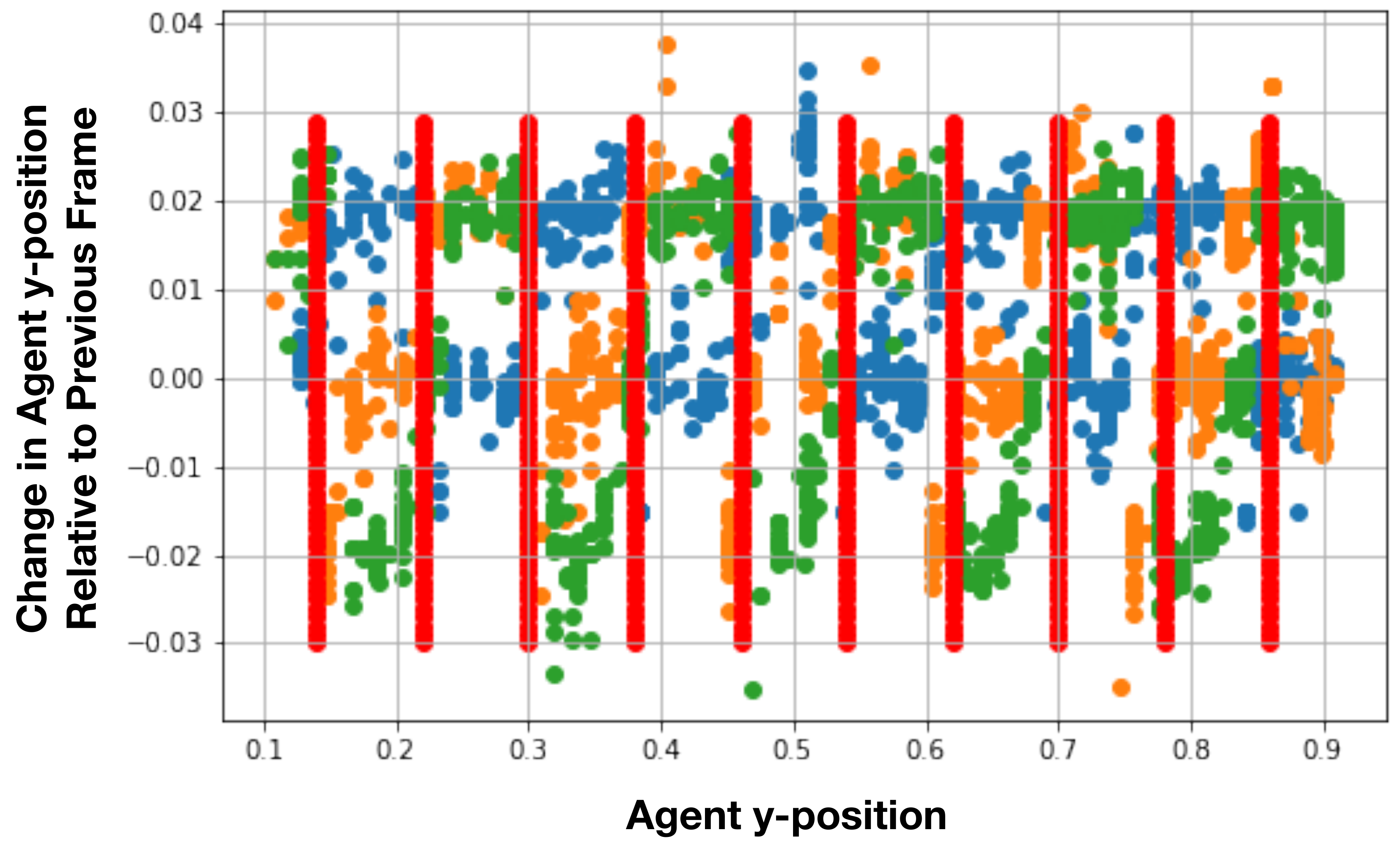}
  \end{center}
  \caption{Visualization of the latent actions learned in Freeway. X-axis is the agent's y-position, y-axis is the displacement of the agent (which is indicative of ground truth action).}
  \figlabel{freeway_scatter}
\end{wrapfigure}

Furthermore, we visualize the learned latent actions for Freeway in \figref{freeway_scatter}. 
In Freeway, the agent can only move along the y-axis, and consequently, the environment action space only has three actions: move up, move down, and no-op. This means that the agent's y-displacement between the current frame and the previous frame directly corresponds to the ground truth action taken. In this visualization, we visualize the the chosen latent action (by color: blue, orange, or green) as a function of the agent's y-position in the current frame (x-axis) and the y-displacement relative to the previous frame (y-axis). 
Note that as mentioned in the paper, we learn the value functions over the top three most dominant actions to stabilize training; for this reason, the visualization only consists of three latent actions.

Within each vertical region marked by the red bars, we see that the latent actions are split into distinct clusters based on the agent's y-displacement. Because the agent's y-displacement directly corresponds to the ground truth action taken, this indicates that given the agent's y-position, the learned latent actions encode information about the ground truth action. Hence, we qualitatively confirm that the state-based purity of the latent actions is high.

\newpage
\subsection{Value Function Details}
\seclabel{vis_value_fn}
\textbf{Value Function Model Selection}
The numbers reported in \tableref{spearmans} are the
95\textsuperscript{th} percentile
Spearman's correlation coefficients over the course of training. 
If training is stable and converges, this corresponds to taking the final value, and 
in the case that training is not stable and diverges, this acts as a form of early stopping.
We take the 95\textsuperscript{th} percentile as opposed to the maximum to eliminate outliers.

Below we present visualizations of value functions learned from LAQ.

\textbf{Freeway.} We visualize the value functions learned using our latent actions for Freeway. \figref{sup_figure_freeway_val} plots the values over the course of an episode. 
In Freeway, the agent has to move vertically up and down the screen to cross a busy freeway, receiving reward when it successfully gets to the other side. In a single episode, the agent can cross the freeway multiple times; each time the agent makes it to the other side, the agent's location is reset to the original starting location, allowing the agent to attempt to cross the freeway once again. For this reason, we see the value increase as the agent gets closer to the other side of the road, and then drop as soon as its position resets to the starting location.
As evident, the peaks of the learned value function correspond highly to the environment reward.
\begin{figure}[h]
    \centering
    \includegraphics[width=\textwidth]{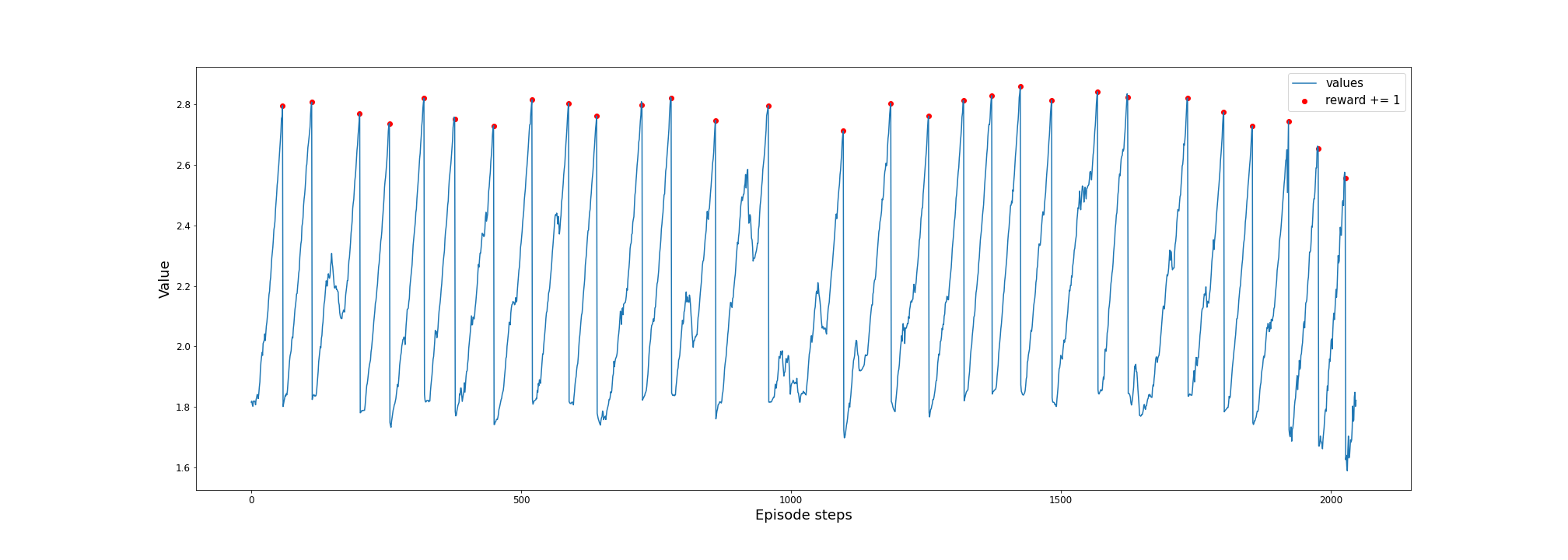}
    \caption{We visualize the value predicted by our learned value function over the course of one episode of Freeway. The red points correspond to when the agent receives a reward from the environment.}
    \figlabel{sup_figure_freeway_val}
\end{figure}

\begin{wrapfigure}[15]{r}{0.5\textwidth}
  \begin{center}
    \includegraphics[width=0.48\textwidth]{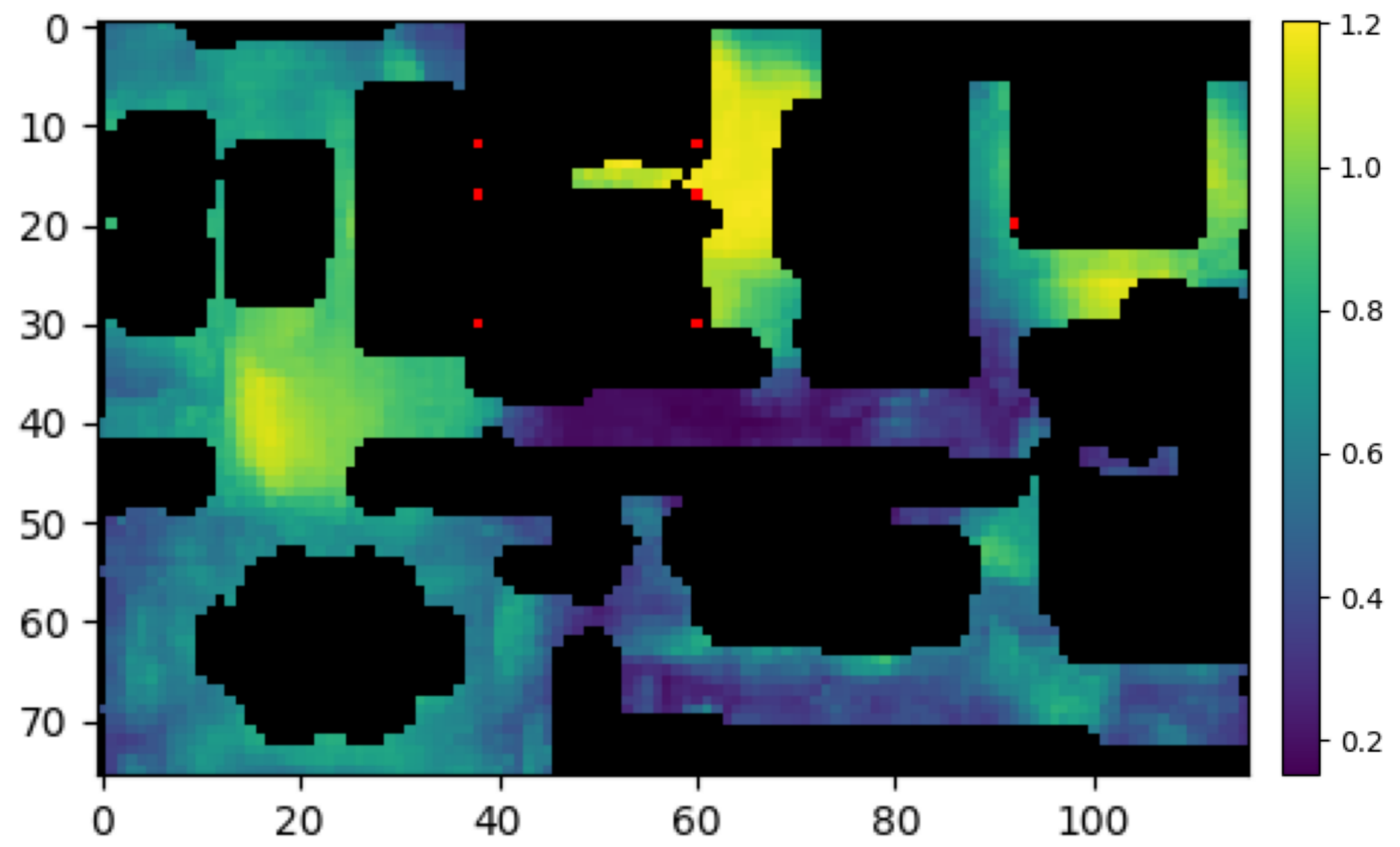}
  \end{center}
  \caption{Visualization of the learned value function for 3D Visual Navigation.}
  \figlabel{sup_figure_vis_nav}
\end{wrapfigure}
\textbf{3D Visual Navigation.}
\figref{sup_figure_vis_nav} visualizes the learned value map in the 3D Visual Navigation branching environment from~\cite{chang2020semantic}. As depicted in \figref{3d-nav}, there are two goals: $G_{near}$ and $G_{far}$. The figure on the right illustrates that the value function learned by utilizing our latent actions correctly assigns high value to the regions surrounding the two goal locations, and low value elsewhere. Additionally, we learn the value functions with DQN using a dataset obtained by using a sub-optimal policy which prefers to go to the goal further away rather than the goal close by. Despite this sub-optimality, the learned value function correctly assigns a higher value to the nearby goal than the goal which is further away.

\newpage

\begin{wrapfigure}{r}{0.5\textwidth}
    \includegraphics[width=0.48\textwidth]{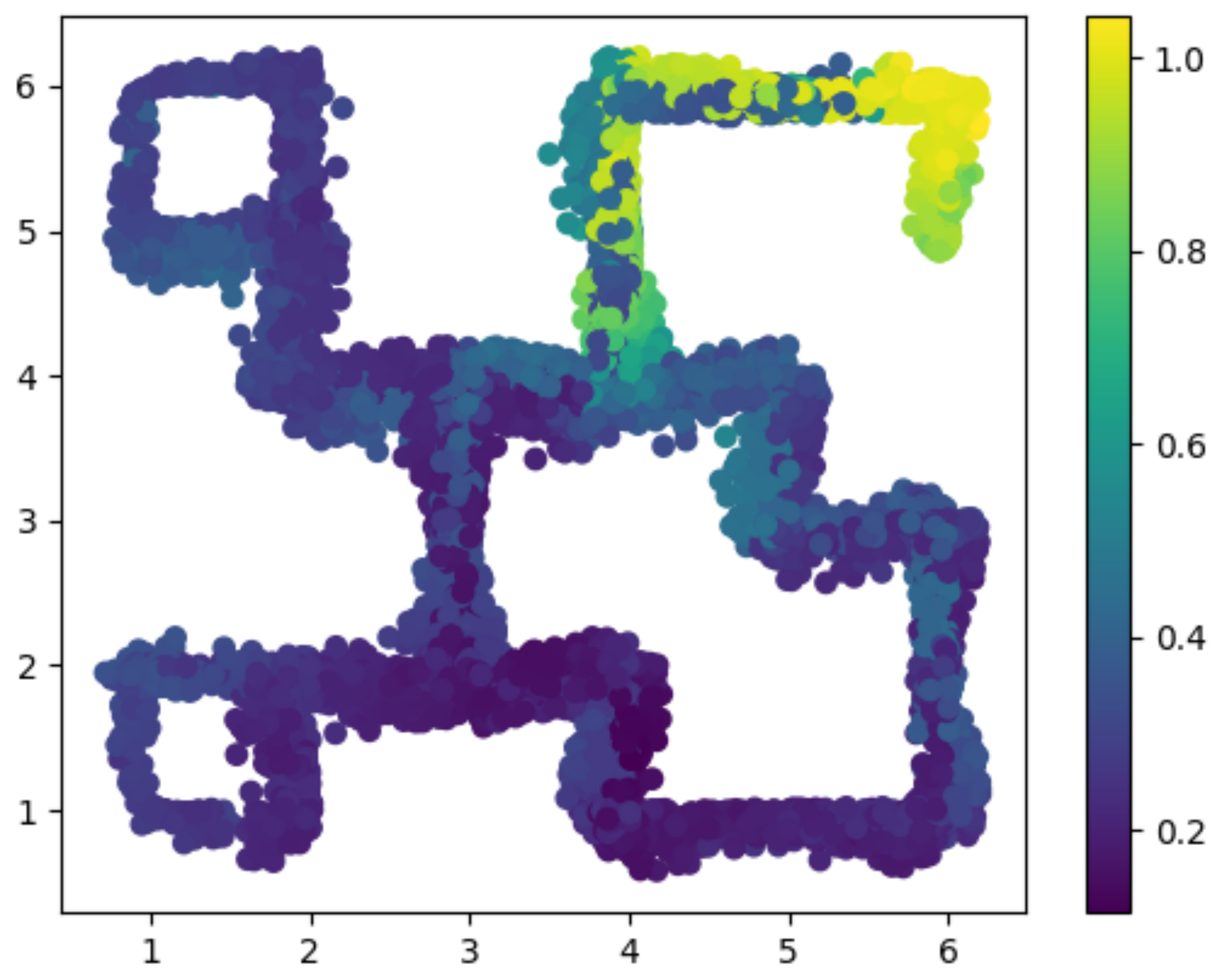}
  \caption{Visualization of the learned value function for 2D Continuous Control.}
  \figlabel{sup_figure_maze2d_val}
\end{wrapfigure}
\textbf{2D Continuous Control.} \figref{sup_figure_maze2d_val} shows a visualization of the value function learned using latent actions for the 2D Continuous Control environment. While the observation space of this environment is $(x,y, v_x, v_y)$ \ie location and velocity of the agent, we produce this visualization over just the location of the agent. Based on just the location of the agent, the optimal value function would be a monotonically decreasing function as the distance from the goal location increases. In this particular environment, the goal of the agent is to get to the top-right corner of the maze. The visualization shows that the learned value function produces high values around this goal region at the top-right corner of the maze, and gradually lower values the farther away you go. This figure visually is in line with our quantitative results in \tableref{spearmans} which show that the value function learned using latent actions in the 2D Continuous Control environment highly correlates to that learned using ground truth actions.

\begin{wrapfigure}{r}{0.5\textwidth}
  \begin{center}
    \includegraphics[width=0.48\textwidth]{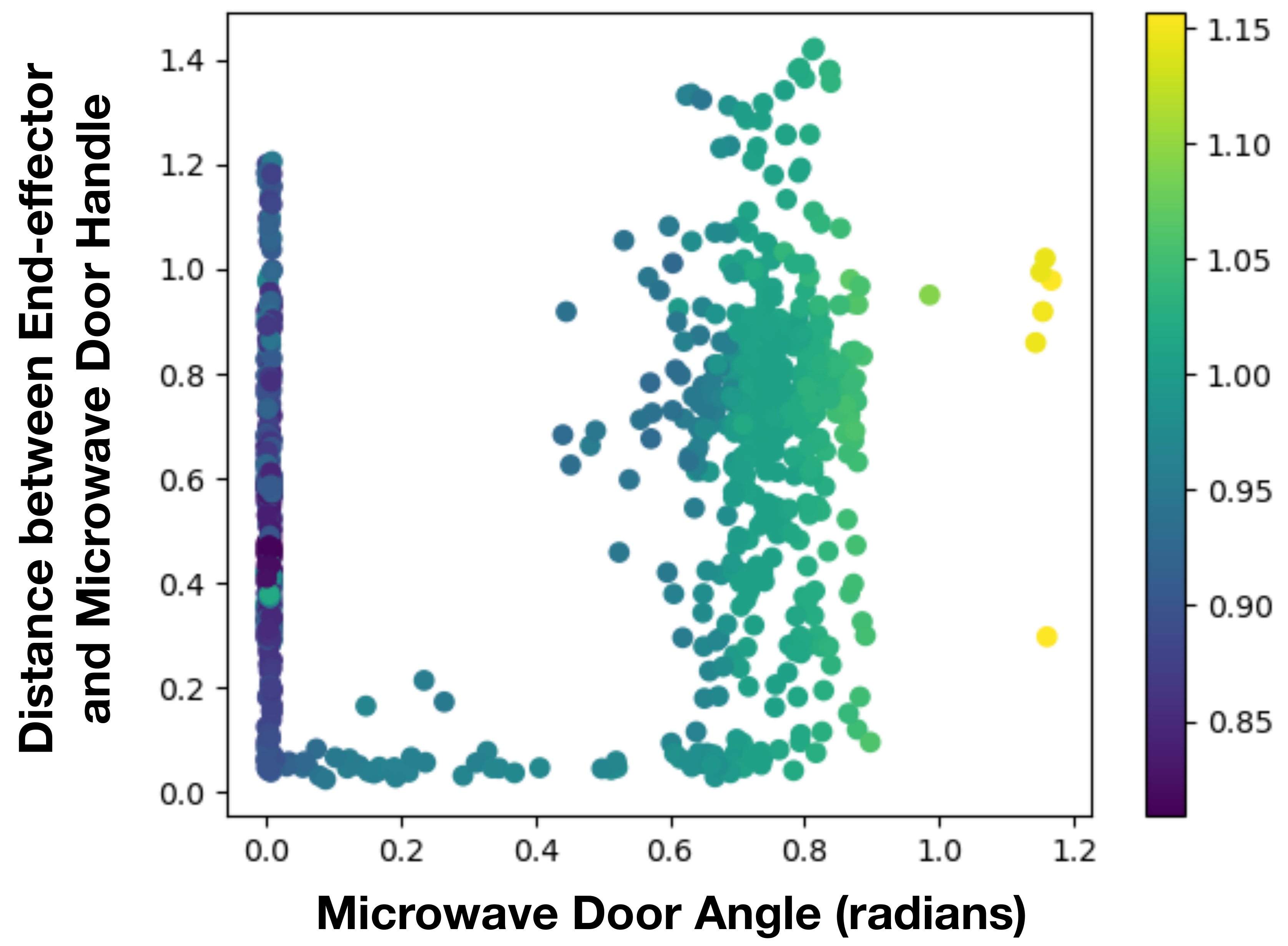}
  \end{center}
  \caption{Visualization of the learned value function in Kitchen Manipulation. The x-axis is the microwave angle in radians, the y-axis is the distance between the end-effector and the microwave door handle, with the colors corresponding to the magnitude of the values.}
  \figlabel{sup_figure_kitchen_heatmap}
\end{wrapfigure}

\textbf{Kitchen Manipulation.} We visualize the values as a function of both the angle that the microwave door makes with respect to its starting state as well as the distance between the end-effector and the microwave door handle in \figref{sup_figure_kitchen_heatmap}. 
The task here is to open the microwave door: the agent receives a binary reward when the angle that the microwave door makes with respect to its starting state (i.e. closed microwave door) is above a threshold (approx. 0.7 radians), and zero otherwise. 
As expected, we see that the predicted value of a state increases as the angle of the microwave door increases. 
The end-effector position does not start out at the microwave door handle, but rather is initialized a fixed distance away from the handle. As a result, the agent must first move the end-effector to the microwave door handle, before interacting with the door handle to open it. 
In addition to this, we hypothesize that as the distance between the end-effector and the microwave door handle decreases, the predicted value should increase.
While not as obvious as the relation between the microwave door angle and the value, we see some indication that as the distance between the end-effector and the microwave door handle decreases, the value increases.

\newpage
\subsection{Dealing with Extrapolation Error} \seclabel{1-norm}
Given the reward structure utilized in densified reinforcement learning (see \secref{behavior}), if
there exist states which produce spuriously high values of $V(s)$, the optimal
policy in the new MDP may erroneously seek these states. Here we present two
ways to combat this phenomenon: \\
\textbf{1)} Fit a density model to constrain
high reward regions to the distribution used to train $V(s)$ \\
\textbf{2)}
Re-parameterize the value function such that out of distribution states are
likely to have low value. 

\textbf{Density Model: } We add an additional component to the reward
computation to incentivize behavior to remain within the region of the state
space covered by the demonstrations (where the learned value function is
in-distribution). For the kitchen manipulation environment we build a density
model over the end effector position using a 2 component Gaussian mixture
model. If a state $s$ has a less than 1\% probability according to the GMM
density model, then we assign $V(s) = 0$, for the update described in
\secref{rl_exps}. We apply this shaping to all methods in \figref{behavior} for
the kitchen environment. We give sparse reward the same benefit of this shaping
by giving reward $-1$ outside of the GMM distribution, and $0$ inside. Sparse
task reward remains the same.

\textbf{Value function re-parameterization} For sparse reward tasks, we can
parameterize the final value prediction as $1-\|\phi(s)\|_2$. Where $\phi$ is a
neural network based featurization of the state. This ensures high values are
only predicted when the featurization of the state is close to the zero vector.
States which are out of distribution may induce random features, but these will
result in lower predicted values.

Results on the kitchen manipulation environment are presented in
\figref{1-norm}. We can see that all methods solve the task poorly in
this environment when value functions are unconstrained. Using a density based
model (GMM) or the ``1-norm'' parameterization of the value function allow the
task to be solved consistently.

\begin{figure}[h]
\centering
\insertH{0.35}{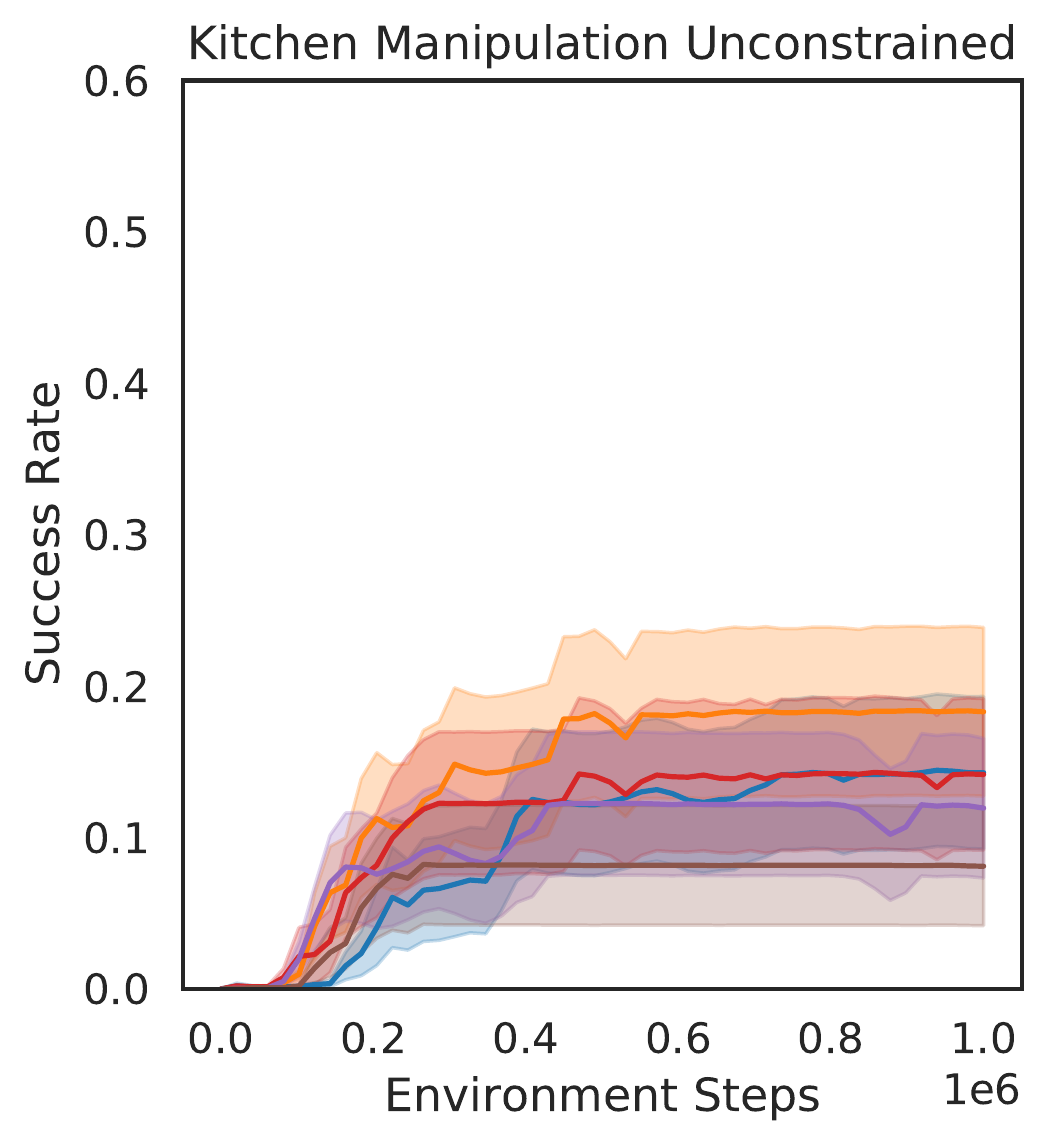} \insertH{0.35}{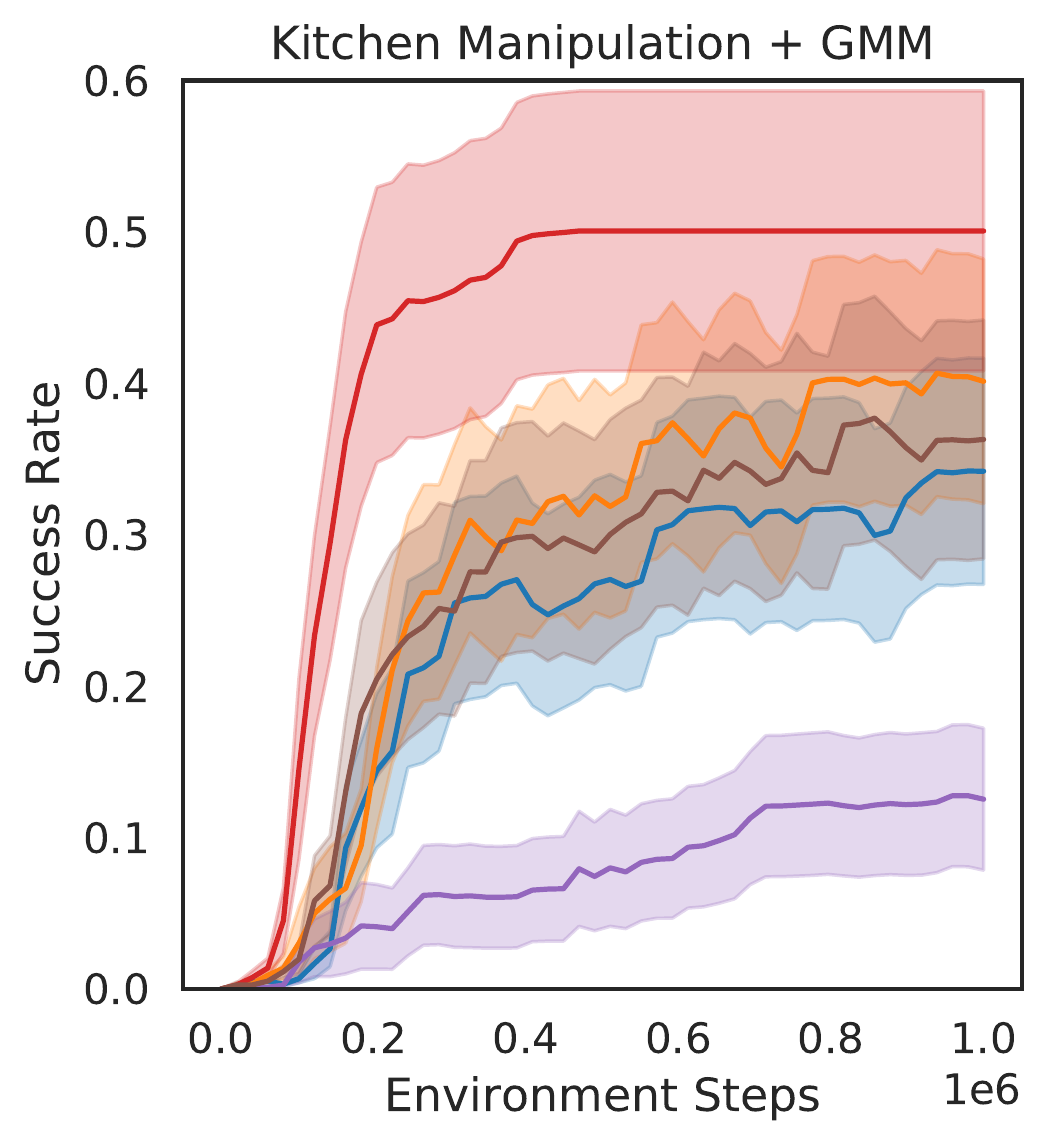} \insertH{0.35}{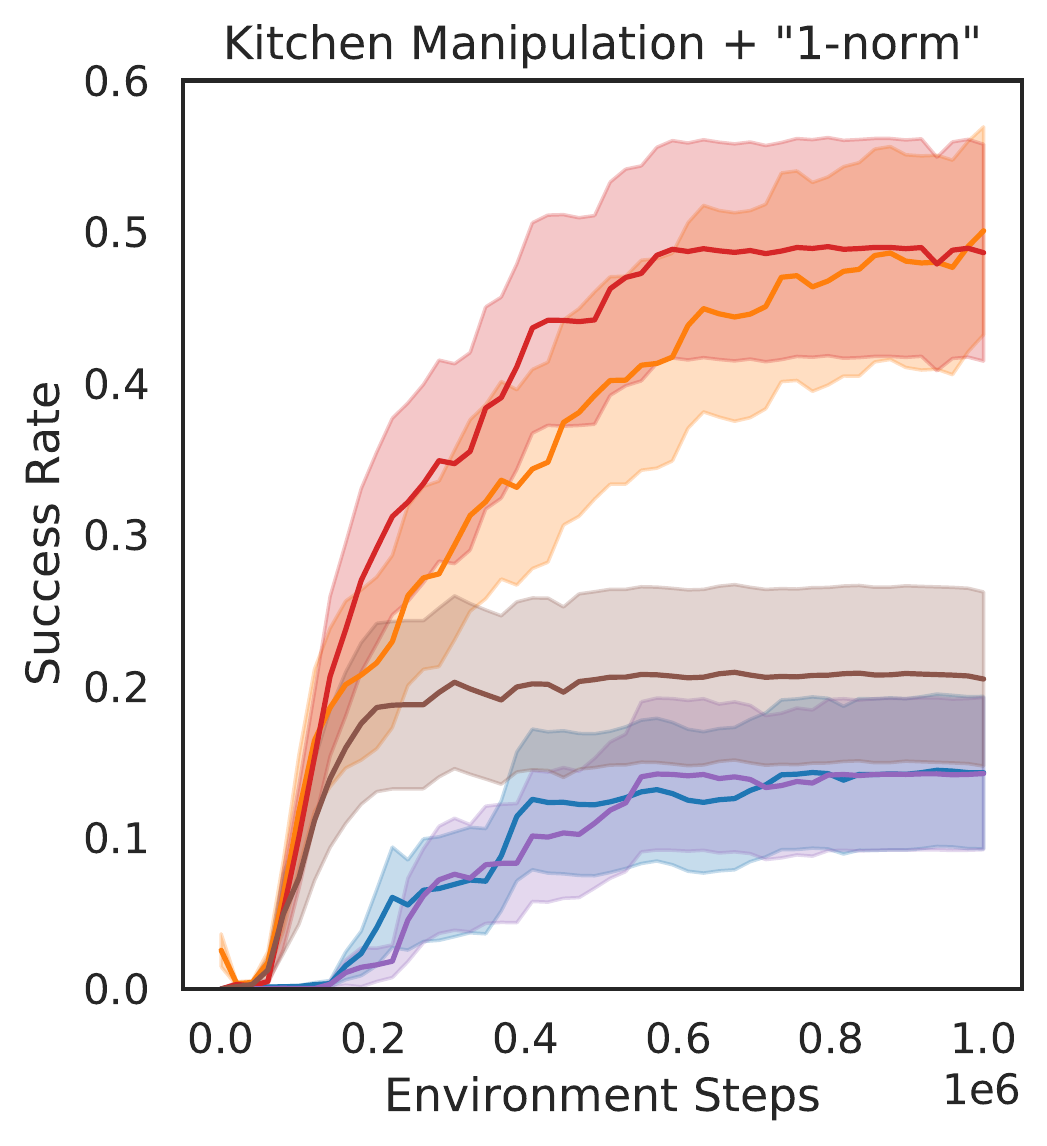}\\
\insertW{1}{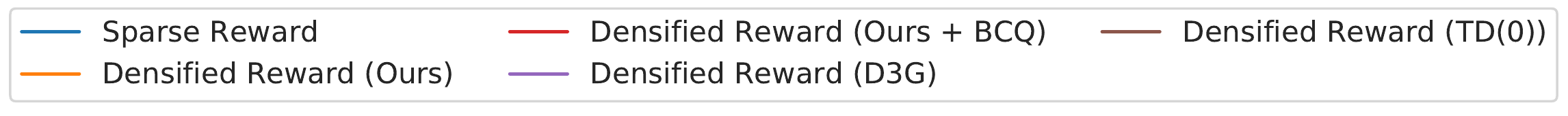}\\
\caption{We show learning curves for acquiring behavior using different methods for dealing with out of distribution samples in the kitchen manipulation environment: unconstrained (left), density model based reward shaping (center), and ``1-norm'' parameterization of the value function (right). Results are averaged over 5
seeds and show 95\% confidence intervals. See \secref{1-norm} for details.}
\figlabel{1-norm}
\vspace{-4mm}
\end{figure}

\end{document}